\PassOptionsToPackage{dvipsnames, svgnames}{xcolor}
\documentclass[sigconf, nonacm]{acmart}

\usepackage{xspace}
\usepackage{framed} 
\usepackage{graphicx}
\usepackage[dvipsnames, svgnames]{xcolor}
\usepackage{soul}
\usepackage{url}
\usepackage{hyperref}
\usepackage{footnote}
\usepackage[english]{babel}
\usepackage{amsmath,amsfonts}
\usepackage{amsthm}

\usepackage{algorithm}
\usepackage[noend]{algpseudocode}
\usepackage{textcomp}
\usepackage{mathrsfs}
\usepackage{booktabs}
\usepackage{enumerate}

\newtheorem{lemma}{Lemma}
\newtheorem{theorem}{Theorem}

\def\BibTeX{{\rm B\kern-.05em{\sc i\kern-.025em b}\kern-.08em
    T\kern-.1667em\lower.7ex\hbox{E}\kern-.125emX}}

\newcommand{\kw}[1]{{\ensuremath {\mathsf{#1}}}\xspace}
\newcommand{\stitle}[1]{ \noindent{\textbf{#1}}}

\newcommand{\ltos}{\kw{L2SP}}
\newcommand{\ltoss}{\kw{L2SPs}}

\newcommand{\squishlisttight}{
 \begin{list}{$\bullet$}
  { \setlength{\itemsep}{0pt}
    \setlength{\parsep}{0pt}
    \setlength{\topsep}{0pt}
    \setlength{\partopsep}{0pt}
    \setlength{\leftmargin}{2em}
    \setlength{\labelwidth}{1.5em}
    \setlength{\labelsep}{0.5em}
 } }

\newcommand{\squishnumlist} {
\newcounter{qcounter}
\begin{list}{\arabic{qcounter}.~}{\usecounter{qcounter}} 
{  \setlength{\itemsep}{0pt}
    \setlength{\parsep}{0pt}
    \setlength{\topsep}{0pt}
    \setlength{\partopsep}{0pt}
    \setlength{\leftmargin}{2em}
    \setlength{\labelwidth}{1.5em}
    \setlength{\labelsep}{1.5em}
}}

\newcommand{\squishend}{
  \end{list}
}

\newcommand\vldbdoi{XX.XX/XXX.XX}

\begin{document}
\title{Path-LLM: A Shortest-Path-based LLM Learning for Unified Graph Representation}

\author{Wenbo Shang}
\affiliation{%
  \institution{Hong Kong Baptist University}
  \city{Hong Kong}
  \state{China}
}
\email{cswbshang@comp.hkbu.edu.hk}

\author{Xuliang Zhu}
\affiliation{%
  \institution{Shanghai Jiao Tong University}
  \city{Shanghai}
  \country{China}
}
\email{zhu.xl@sjtu.edu.cn}

\author{Xin Huang}
\affiliation{%
  \institution{Hong Kong Baptist University}
  \city{Hong Kong}
  \country{China}
}
\email{xinhuang@comp.hkbu.edu.hk}


\begin{abstract}
Unified graph representation learning aims to {generate} node embeddings, which can be applied to multiple downstream applications of graph analytics. 
However, existing studies based on graph neural networks and language models either suffer from the limitations of numerous training needs toward specific downstream predictions, poor generalization, or shallow semantic features. In this work, we propose a novel Path-LLM model to efficiently learn unified graph representation, which leverages a powerful large language model (LLM) to incorporate our proposed path features. Our Path-LLM framework consists of four well-designed techniques. First, we develop a new mechanism of long-to-short shortest path (\ltos) selection, which can cover key connections between different dense groups. An in-depth analysis and comparison of different path selections is conducted to justify the rationale behind our designed \ltos method. Next, we design path textualization to obtain L2SP-based training texts with key phrase selection from node text attributes. We then feed the texts into {a self-supervised LLM training process to align next node/edge generation in L2SP with next token generation in causal language modeling for graph representation learning} and finally extract the unified graph embeddings.
We theoretically analyze the \emph{algorithm complexity} of our Path-LLM approach. Extensive experiments on large-scale graph benchmarks validate the superiority of Path-LLM against state-of-the-art methods \emph{WalkLM}, \emph{GraphGPT}, \emph{OFA}, and \emph{GraphTranslator} on two classical graph learning tasks (node classification and edge validation) and one NP-hard graph query processing task (keyword search). Compared with WalkLM, our approach \emph{saves more than 90\% of training paths on millions-scale graphs} and \emph{runs at most 35\textbf{x} faster}. Besides the quality and efficiency evaluations, we have also conducted several \emph{ablation studies}, \emph{case studies}, and \emph{embedding visualizations} to show the effectiveness of Path-LLM.

\end{abstract}

\maketitle




\section{INTRODUCTION}
  \begin{figure}
  \includegraphics[width=0.47\textwidth]{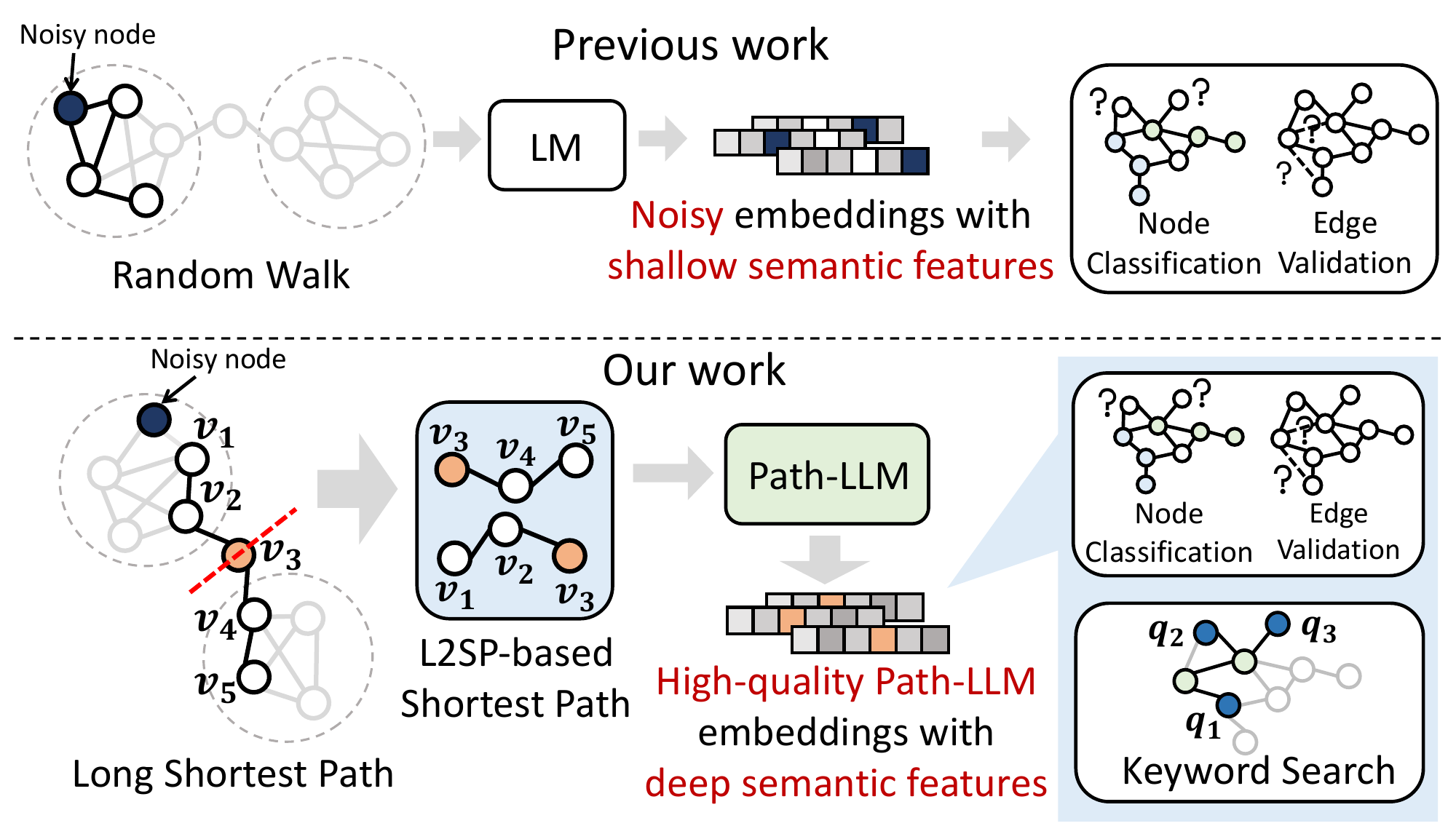}
  \caption{A comparison of existing works and our work. }
  \label{fig:intro}
\end{figure}
Graphs play a crucial role in various real-world application scenarios, including academic networks, biomedical graphs, social networks, financial networks, and so on~\cite{zhou2020query, fang2022densest, sen2008collective, wu2018moleculenet, broder2000graph}. These graph data associated with textual information, such as the attributes of nodes and edges, representing complex and diverse semantics, are commonly known as text-attributed graphs (TAGs)~\cite{yang2020scaling,kochsiek2021parallel, li2025grela}. A wide range of graph representation learning models has been proposed, including successful graph neural networks (GNNs)~\cite{zhang2022benchmark, gcn, velivckovic2017graph, gilmer2017neural}. However, these models often require sufficient training toward specific downstream predictions~\cite{zhu2021graph, walklm} and focus more on processing graph structure, ignoring rich semantics.
For obtaining graph representation with richer semantics, existing path-based works~\cite{walklm} have recently proposed to integrate \emph{LMs} and \emph{random walks} to derive unified graph embeddings for multiple graph learning tasks, e.g., node classification and edge validation. However, it still suffers from the following three limitations. 
1) The sampled random walks can hardly cover the bridge edges between different dense groups, as the path of the random walk has a high probability of falling within one dense subgraph; 2) random walks can easily involve noisy nodes to damage the embedding quality. Meanwhile, the node texts of irregular paths cannot match the linguistic rules, bringing more difficulty for learning models, and 3) the transferability of existing path-based works to LLMs with causal language modeling is limited, as it is primarily designed for small LMs with masked language modeling, such as BERT~\cite{devlin2019bert} and DeBERTa~\cite{hedeberta}, which generally have weaker performance than the recent LLMs~\cite{trustLLM}.

To tackle the above limitations, we propose a novel graph representation learning model of Path-LLM, as shown in Figure~\ref{fig:intro}. Our developed Path-LLM utilizes \emph{powerful LLMs with causal language modeling}, \emph{well-designed shortest-path-based graph features} and \emph{a shortest-path-based self-supervised learning method} to learn unified graph embeddings for several downstream tasks, including node classification, edge validation, and one typical NP-hard graph analytics task of keyword search~\cite{wang2010survey, li2008ease}. 
Path-LLM enjoys two major advantages based on \emph{LLMs} and \emph{shortest paths}. First, the advent of LLMs (e.g., GPT~\cite{gpt3,gpt4} and Llama~\cite{llama2}) has revolutionized the field of language processing. LLMs possess billions of parameters and have been trained on extensive corpora, enabling them to exhibit powerful semantic representation capabilities and strong generalization abilities~\cite{instructGPT, zhuang2024toolqa, zhang2023prompting}. 
Second, the shortest path is the closest path between two nodes, which can avoid unnecessary detours through noisy nodes and cycles, in contrast to random walks. The shortest paths usually indicate regular paths to follow in a complex graph, indicating good-quality learning features for LLMs. {Third, shortest paths are suitable for causal language modeling in LLMs, as searching the next node based on previous nodes is more aligned with generating next tokens based on prefix tokens, against masked language modeling in existing works.}

It is challenging to achieve a diverse and high-quality set of shortest paths and enable LLMs to understand complex attributes and graph structures in self-supervised manners. To tackle it, we construct a new path extraction mechanism to select training features, called the long-to-short shortest paths. 
It first samples a few long shortest paths to capture both cross-group connections and internal group structures and then cuts them into short paths in designed ways for Path-LLM effective and efficient learning. 
In addition, 
we design a path textualization function to transform  \ltos-based structural information into \ltos-based texts for Path-LLM learning. To ease the Path-LLM learning process, we conduct data cleaning and key phrase selection based on PositionRank~\cite{florescu2017positionrank} to reconstruct new text attributes of nodes during path textualization.
To embed graph structures into the semantic space, we 
feed \ltos-based texts into Path-LLM for a self-supervised pre-training process. 
During the training process, we align L2SP-based shortest paths with the LLM-learned linguistic rules.
As Path-LLM learns the order of tokens in the \ltos-based text, it also learns the order of nodes and edges within the L2SP, thereby learning graph structures.
Finally, we derive an integrated embedding for all nodes from the frozen Path-LLM, which is effective for several downstream tasks. To summarize, we make the following contributions:
\begin{itemize}
\item We propose a novel Path-LLM model for generating unified graph embeddings, which can effectively handle both homogeneous and heterogeneous graphs for multiple downstream graph analytics tasks. 
(Section~\ref{path-LLM})
\item We first propose the long-to-short shortest paths (\ltos) and design a path textualization function to construct \ltos-based texts.
Moreover, we conduct a comprehensive analysis by comparing our \ltos method against different path selections to show the advantages and suitability of short paths and \ltos selection in Path-LLM learning. (Section~\ref{sec:l2sp textualization})
\item {We propose a new Path-LLM graph feature learning method for graph embedding learning, which aligns next node/edge generation in L2SP with next token generation in causal language modeling. We then extract Path-LLM embeddings for downstream tasks. }
Furthermore, we give a complexity analysis of Path-LLM, and provide theoretical proof demonstrating that Path-LLM can learn the graph structure through self-supervised learning of L2SP-based texts. (Section~\ref{sec:training and embedding})
\item To illustrate the usefulness of our learned embedding results, we revisit a useful but challenging graph task of keyword search~\cite{keywordsearch}, which finds a subgraph covering all keywords with tight closeness of topology structure and node semantics. We develop a method of weighted TAG construction based on the Path-LLM embedding vectors and introduce approximate solutions.  (Section~\ref{new-task}) 
\item We conduct extensive experiments to validate the effectiveness of Path-LLM outperforming the state-of-the-art WalkLM method on four real-world benchmarks. Our Path-LLM model only uses 9\% of WalkLM's training paths on average. We also conduct a case study of keyword search on PubMed and embedding visualization to show the usefulness of our Path-LLM model. (Section~\ref{experiment})  
\end{itemize}
We give preliminaries 
in Section~\ref{sec:preliminary}. We review related work in Section~\ref{sec:related work} and conclude the paper in Section~\ref{sec:conclusion}. 


\section{RELATED WORK}
\label{sec:related work}
In recent years, various graph representation learning approaches have been studied 
\cite{shang2024survey, ju2024comprehensive, li2024glbench, li2023survey, wang2022survey}. 

\stitle{GNN-based graph representation learning.}
Graph neural networks (GNNs) have been employed to learn node representations by aggregating information from neighboring nodes on graphs \cite{zhou2020graph,wu2020comprehensive,graphsage,gcn, li2020distance, zhu2020deep, velivckovic2017graph, gao2021ics}. Some studies~\cite{hu2020gpt,jiang2021pre,luo2023self,you2020graph} have utilized self-supervision techniques to pre-train a sophisticated GNN model for unified graph embedding learning, such as contrastive learning. However, this method predominantly focuses on the graph structure, leading to shallow and rough alignment of semantic information and insufficient integration of structural and semantic information.

\begin{figure*}
  \includegraphics[width=0.98\textwidth]{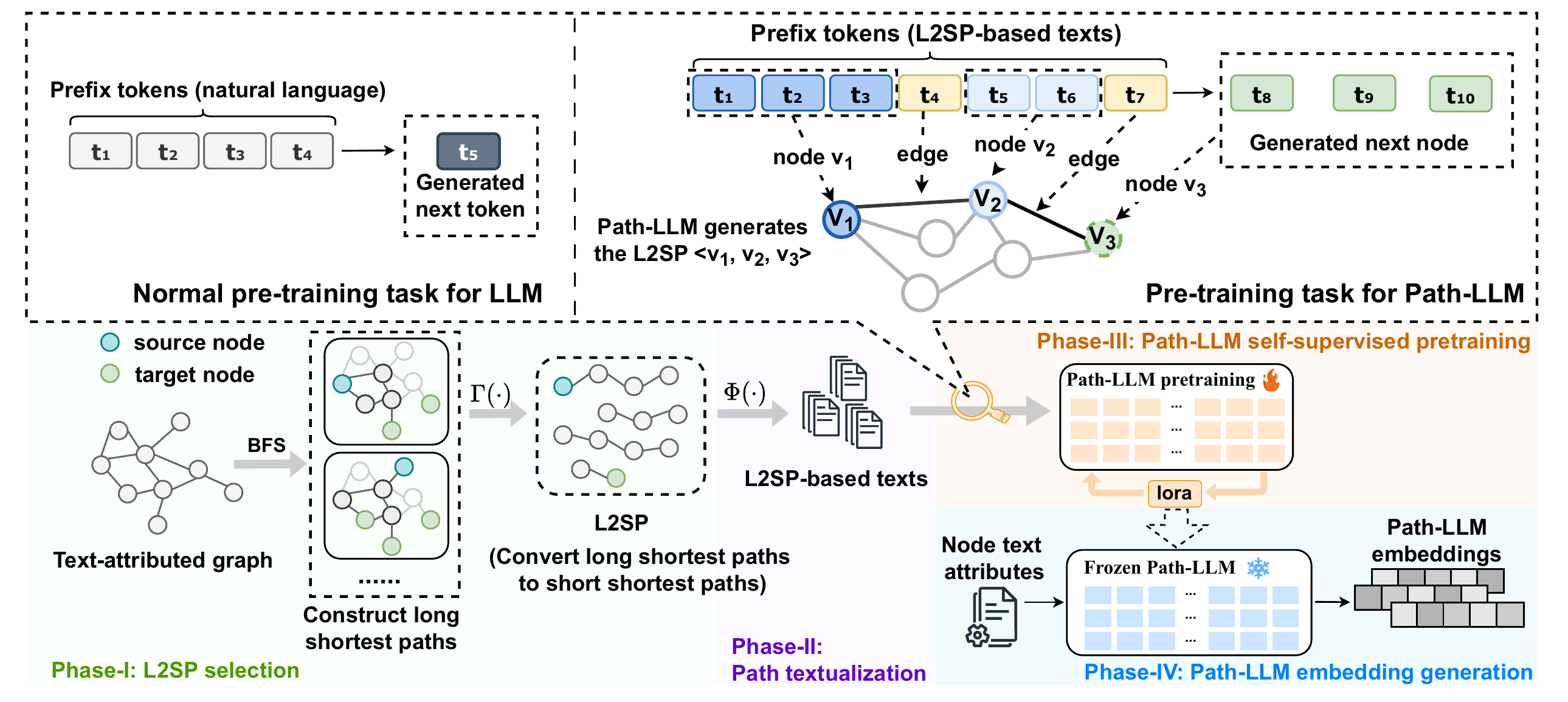}
  \caption{Our proposed Path-LLM framework involves four key components: (1) \underline{Long-to-short shortest path (L2SP) selection}, converting sampled long shortest paths to L2SP-based shortest paths. (2) \underline{Path textualization}, transforming L2SP-based shortest paths to L2SP-based texts. (3) \underline{Path-LLM pre-training}, learning graph structures from L2SP-based texts. (4) \underline{Path-LLM embedding generation}, deriving the final embeddings from the frozen Path-LLM.}
  \label{fig:framework}
\end{figure*}

\stitle{LM-based graph representation learning.}
Several studies have employed diverse methodologies to generate graph embeddings using Language Models (LM) \cite{li2023grenade,wang2023dynamic,herzig2023incorporating,zerog,walklm, he2023harnessing, pan2024survey}. 
Current methodologies for LLM-based graph representation can be classified into two main categories: supervised and self-supervised graph representation methods.
Numerous cutting-edge supervised techniques are specifically designed for single graph learning tasks~\cite{ye2024language}, such as node classification, exemplified by TAPE~\cite{he2023harnessing} and GraphAd-apter~\cite{huang2024can}. Furthermore, further supervised approaches train their models across one more task~\cite{chen2024llaga}, such as OFA~\cite{liu2024one} and GraphTranslator~\cite{zhang2024graphtranslator}. Unlike supervised methods, which require extensive labeled data and are specifically trained for particular downstream tasks, self-supervised methods do not need labeled data or human manual resources.
Currently, the state-of-the-art approach, WalkLM \cite{walklm}, still utilizes LM to generate embeddings, which is markedly inferior to that of LLM \cite{instructGPT,trustLLM}, such as GPT-4 \cite{gpt4} and Llama \cite{llama2}. 
Although GraphGPT~\cite{tang2024graphgpt} employs a self-supervised learning approach to adapt LLMs for various downstream tasks, it primarily focuses on enhancing question-answering and reasoning capabilities rather than deriving unified graph embeddings.
\emph{Different from most existing studies, our emphasis is to generate unified embeddings for each node to capture both its topological and semantic features. Our learning model leverages advanced LLMs and well-developed shortest path selection to tackle graph learning tasks and one NP-hard graph querying task of keyword search.}

\section{PRELIMINARY}
\label{sec:preliminary}
In this section, we introduce the causal language modeling in LLMs, notations, and the objective of our Path-LLM model.
\subsection{Causal Language Modeling in LLMs}
For the pre-training process, LLMs mainly use Causal Language Modeling (CLM) training techniques, which are trained to predict the next token $x_i$ in a sequence $x=\{x_1,x_2,...,x_q\}$ based on prefix tokens $x_{<i}=\{x_1,x_2,...,x_{i-1}\}$. CLM is commonly used to train LLMs like GPT~\cite{gpt3,gpt4} and Llama~\cite{llama2}. 
LLM is typically trained to optimize a conditional probability distribution $p(x_i|x_{<i})$~\cite{he2023harnessing}, which assigns a probability to each possible $x'_i \in \mathcal{D}$ given prefix tokens $x_{<i}$, where $\mathcal{D}$ is the LLM vocabulary. 
Thus, the probability of the output sequence $x$~\cite{he2023harnessing, yadkori2024believe, requeima2024llm} can be formulated as :
\begin{equation}
\label{llmprob}
    p(x) = \prod\limits_{i=1}^q p(x_i|x_{<i}).
\end{equation}
Notably, the probability of generating token $x_i$ depends only on the prefix tokens $x_{<i}$, showing that LLMs are blind to the following tokens $x_{>i}$ after $x_i$.

\subsection{Problem Formulation}
\stitle{Text-attributed graphs.} A \underline{t}ext-\underline{a}ttributed \underline{g}raph (TAG) can be represented as $G=(V,E,\mathcal{X}_v,\mathcal{X}_e)$, where $V$ and $E$ denote the set of nodes and edges. $V=\{v_1,v_2,...,v_n\}$ is the set of $n$ nodes paired with raw text attributes $\mathcal{X}_v=\{\mathcal{X}_{v_1},\mathcal{X}_{v_2},...,\mathcal{X}_{v_n}\}$. $E=\{e_{i,j}\}$ is the set of edges where $e_{i,j}$ is an edge from $v_i$ to $v_j$ and is paired with raw text attributes $\mathcal{X}_e=\{\mathcal{X}_{e_{i,j}}\}$. The path in a TAG can be defined as $\mathcal{P}=\langle v_1, v_2,...,v_\ell\rangle$, such that $v_i$ is adjacent to $v_{i+1}$ for $1 \leq i < \ell$. Such a path $\mathcal{P}$ is called a path of length $\ell$ from $v_1$ to $v_\ell$.
Furthermore, the shortest path from $s$ to $t$ is the path $\mathcal{P}=\langle v_1, v_2,...,v_\ell\rangle$ (where $v_1=s$ and $v_\ell=t$) that over all possible paths between these two nodes with the fewest edges.


\stitle{Objective of Path-LLM.} Given a text-attributed graph $G$, the~goal of Path-LLM is to generate unified graph embeddings $\xi$ integrating complex graph structures and text attribute semantics in $G$, and then improve task performances on multiple downstream tasks (e.g., node classification, edge validation and keyword search). Specifically, each node $v_i \in V$ is paired with the embedding $\xi_{v_i}$ extracted from Path-LLM. {Note that, different from supervised methods, we focus on self-supervised graph representation learning methods, where models learn by solving pretext tasks, with supervision signals automatically derived from the data itself~\cite{liu2022graph, xie2022self,liu2021self}.}

\section{An  Overview of Path-LLM}
\label{path-LLM}
In this section, we introduce our Path-LLM model, which learns graph structures through \emph{our proposed mechanism of L2SP selection} and \emph{LLM self-supervised pre-training}. 
Figure \ref{fig:framework} depicts the framework of Path-LLM with four key components in different phases.


\squishlisttight
\item \textbf{Phase-I: L2SP selection}. To capture comprehensive graph properties, we first sample a few shortest paths of long length widely across the whole network and then cut them into short paths. We construct these long-to-short shortest paths (\ltos) as a base set of important path-based features for Path-LLM.  
\item \textbf{Phase-II: Path textualization}. Next, we employ the path textual function to obtain \ltos-based texts, incorporating the properties of \ltos and key phrases in text attributes, with each text representing a single \ltos within the graph.
\item \textbf{Phase-III: \ltos-based graph feature learning.} These \ltos-based texts form a comprehensive dataset and are subsequently fed into the Path-LLM for self-supervised learning. As the Path-LLM acquires the ability to generate \ltos-based texts, it inherently learns to generate \ltos-based shortest paths. 
\item \textbf{Phase-IV: Path-LLM embedding generation}. Finally, we utilize the frozen Path-LLM model to extract Path-LLM embeddings that combine shortest path-based graph features with deep semantic information derived from node text attributes.
\end{list}

\section{L2SP Selection and Textualization}
\label{sec:l2sp textualization}
This section emphasizes the methodology for transforming the graph into a textual format, which provides Path-LLM with a high-quality graph training corpus. This conversion maintains the structural features and essential semantic information from the text attributes of nodes and edges.
\subsection{Phase-I: L2SP Selection}
{Generally, a graph has complex structures composed of dense groups and interconnections among them. 
Capturing features of these dense groups and connections is vital.
Hence, to effectively obtain the critical structure of dense groups and their cross-over connections, we first sample long paths $\mathcal{P}_{long}$ from $G$ and then cut $\mathcal{P}_{long}$ into small ones for easy learning by Path-LLM in Algorithm~\ref{algo:l2sp}.

\begin{algorithm}[t]
\small
  \caption{L2SP Selection}
  \label{algo:l2sp}
  \begin{algorithmic}[1]
    \Require A text-attributed graph $G=(V, E, \mathcal{X}_v, \mathcal{X}_e)$, the number of sampling nodes $b$, the minimum long path length $L$, the number of sampling shortest paths $k$, the maximum short path length~$\ell$.
    \Ensure $\mathbb{P}_{short}$ shortest paths that have no length greater than $\ell$.
    \State Randomly select a $b$-sized node set $\mathcal{S}$;
    \For {$s_i \in \mathcal{S} $}
        \State $\mathcal{T}_i\leftarrow \{\tau_i|dist(s_i, \tau_i)\geq L,\tau_i\in V\}$ by BFS;
        \State Randomly select a target node $\tau_i$ in $\mathcal{T}_i$;
        \State $\mathbb{P}_{long}\leftarrow $ Select $k$ shortest paths from $s_i$ to $\tau_i$ by BFS;
    \For{$\mathcal{P}_{long}\langle v_1, v_2,..., v_L\rangle \in \mathbb{P}_{long}$}
        \State $\mathcal{P}_{short} = \{\langle v_{i+1}, ..., v_{i+\ell}\rangle | i = \alpha(\ell-1), 0 \leq \alpha < \lfloor \frac{L}{\ell - 1} \rfloor - 1\} \cup \{\langle v_{(\lfloor \frac{L}{\ell - 1} \rfloor - 1)(\ell - 1) + 1}$ $, ..., v_{L}\rangle\}$;
        \State $\mathbb{P}_{short} \leftarrow \mathbb{P}_{short} \cup \mathcal{P}_{short}$;  
    \EndFor
    \EndFor
    \State \Return{$\mathbb{P}_{short}$};
  \end{algorithmic}
\end{algorithm}

\stitle{Long-path sampling}. Randomly selecting a pair of nodes with a long shortest path between them is highly challenging, as many graphs have a small-world property. To effectively obtain $\mathcal{P}_{long}$, we propose a \emph{fast long-path sampling algorithm}. 
We first randomly select a set $\mathcal{S}=\{s_i\}_{i=1}^b$ of $b$ nodes as the source nodes (line 1). 
A full-graph breadth-first search (BFS)~\cite{bfs} is then conducted for each source node $s_i$ to identify nodes in $G$ at a distance greater than $L$ edges from $s_i$ (lines 2-3). These identified nodes collectively form a candidate set $\mathcal{T}_i=\{\tau_{ij}\}_{j=1}^{b'}$ of target nodes, with the size of $\mathcal{T}_i$ being $b'$. 
In other words, each node in $\mathcal{T}_i$ is beyond the $L$-hop distance from $s_i$. Within this candidate set $\mathcal{T}_i$, a node is randomly chosen as the target node, denoted as $\tau_i$. 
The BFS algorithm is then employed to find a shortest path union cover between $s_i$ and $\tau_i$. 
To reduce the overlapping of these long paths, a $k$-size subset of all long shortest paths with the same $s_i$ and $\tau_i$ is randomly selected to constitute the result $\mathbb{P}_{long}$ (lines 4-5). 
\\
\textbf{Long-to-short path conversion.} 
To convert sampled long shortest-paths $\mathcal{P}_{long}$ into short ones, we propose the \ltos conversion method to cut them into short shortest-paths
(lines 6-8). 
 Given one long shortest path $\mathcal{P}_{long}=\langle v_1, v_2,..., v_L\rangle$
 and a parameter of maximum length $\ell$, i.e., $\mathcal{P}_{short} = \{\langle v_{i+1}, ..., v_{i+\ell}\rangle | i = \alpha(\ell-1), 0 \leq \alpha < \lfloor \frac{L}{\ell - 1} \rfloor - 1\} \cup \{\langle v_{(\lfloor \frac{L}{\ell - 1} \rfloor - 1)(\ell - 1) + 1}$ $, ..., v_{L}\rangle\}$. 
For example, 
{assume that a $\mathcal{P}_{long}=\langle v_1, v_2, v_3, v_4, v_5, v_6, v_7, v_8, v_9, v_{10}\rangle$ and $\ell=3$, the output of $\Gamma(\mathcal{P}_{long})$ is a set of five shortest paths  as $\mathcal{P}_{short}=\{\langle v_1, v_2, v_3\rangle,\langle v_3, v_4, v_5\rangle,\langle v_5, v_6, v_7\rangle,$ $\langle v_7, v_8, v_9\rangle,\langle v_9, v_{10}\rangle\}$. }
Note that these five short paths can be connected via the cutting nodes, e.g., $v_3$, $v_5$, $v_7$, and $v_9$.

\begin{figure}[t]
 \hspace{-0.3cm} \includegraphics[width=0.5\textwidth]{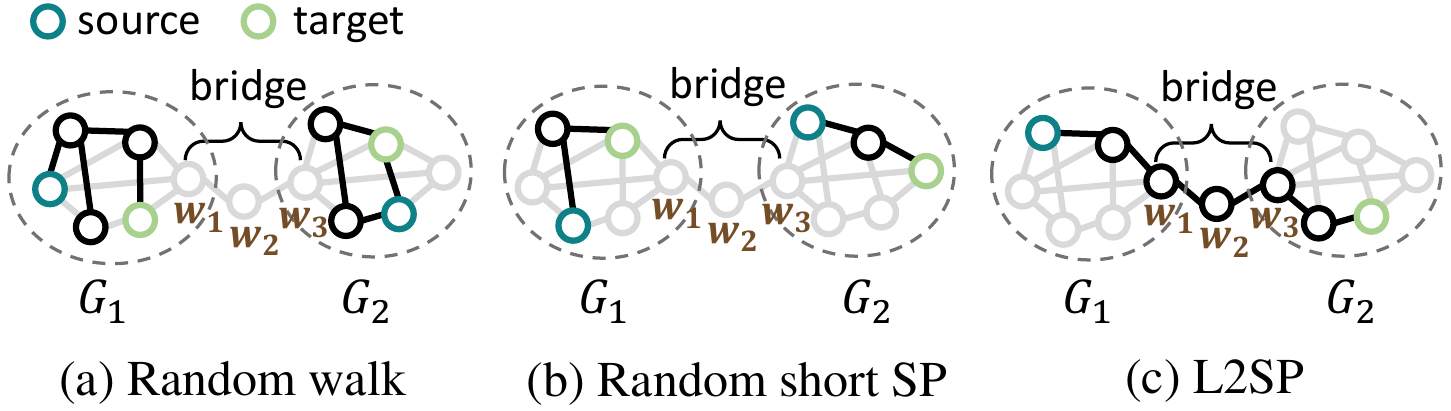}
  \caption{Comparison between random walks, randomly sampled short shortest paths, and \ltos-based paths.}
  \label{fig:long to short}
\end{figure}

\begin{figure*}[t]
    \centering
  \includegraphics[width=0.90\textwidth,height=0.24\textwidth]{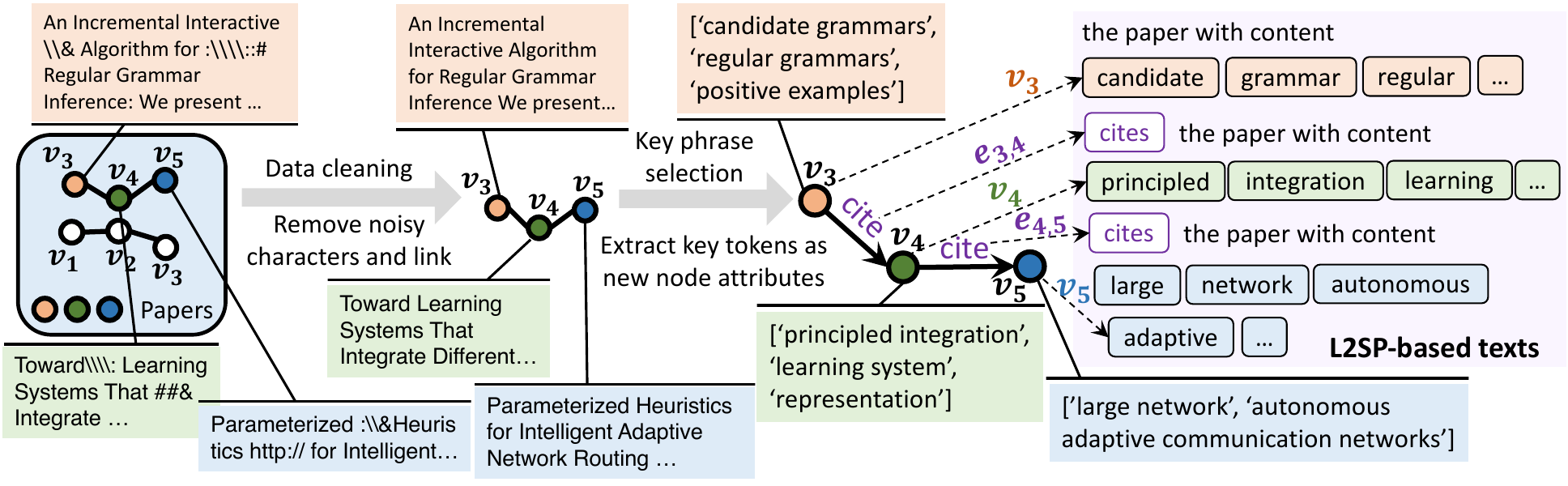}
  \caption{
  Path textualization of homogeneous graphs on citation networks.}
  \label{fig:textual func}
\end{figure*}

\subsection{Comparing Path Selections in Phase-I}
We compare different choices of path selection and analyze their~pros and cons, including 1) random walk v.s. shortest paths, 2) long paths v.s. short paths,  and  3) random short paths v.s. \ltos-based paths. 
\label{text-with-properties}

\stitle{Random walk vs. Shortest paths.
} 
Previous work~\cite{walklm} applies random walk to sample paths, while it is less effective than our \ltos method in LLM due to three major reasons, \emph{noisy and fragile}, \emph{mismatching with LLM}, and \emph{hardly covering bridge edges}.


First, 
{random walk can easily involve noisy nodes and cycles,} rendering them less robust. 
In contrast, the shortest path focuses on direct connections, \emph{avoiding unnecessary detours through noisy nodes and cycles.}
Thus, our methods demonstrate greater resilience to noisy nodes and other disruptions.

Second, random patterns may not match the order of language sequences in LLM. \emph{LLM may have limited ability to learn graph features from irregular random walks.} Most LLMs use causal language modeling (CLM) training techniques. LLMs generate tokens one by one according to the linguistic rules. The next token prediction is only based on the prefix tokens, not the following ones. 
Considering the prefix nodes in random walks, the next node is essentially random and unpredictable. Conversely, \emph{the next node selection in the shortest path is also based on the prefix nodes}, making it more in line with the generation pattern of LLMs.

Last, 
{random walk may hardly cover bridge edges,} i.e., the edges across distinct dense groups. Consider a bridge formed by $w_1-w_2-w_3$ 
to connect two dense groups $G_1$ and $G_2$ shown in Figure~\ref{fig:long to short}.
For a random source node $s$ in a dense group, the generated random walk paths tend to stay within its dense groups, which misses the bridge $w_1-w_2-w_3$ as shown in Figure~\ref{fig:long to short}(a). 
On the other hand, \ltos-based shortest paths are converted from a long shortest path, where the source and target nodes that span a long distance may have a high probability of appearing in different dense groups, as shown in Figure~\ref{fig:long to short}(c). This easily enables to cover the bridge between dense groups by our \ltos-based path selection.

\stitle{Long paths vs. Short paths.
} 
We compare two kinds of shortest paths in terms of different lengths, called long paths and short paths, respectively. Given a long path $\mathcal{P}=\langle v_1,v_2,...,v_\ell\rangle$, the distance between the source node $v_1$ and the target node $v_\ell$ is $\ell$. LLM tends to learn from all the prefix nodes $\langle v_1,...,v_{\ell-1}\rangle$ and generate the next node $v_\ell$. A pair of nodes $(v_1, v_\ell)$ more than $\ell$ hops apart in a graph has a very weak relationship. However, 
the node embedding of $v_\ell$ still may be affected by \emph{the weakly associated node} $v_1$ because LLM generates the next token based on all forward nodes. Therefore, compared to long paths, we chose short shortest paths instead to avoid such irrelevant propagation.

\stitle{Random short paths vs. L2SP-based paths.
}
Similar to random walk, most randomly sampled shortest paths tend to fall within dense groups in Figure~\ref{fig:long to short}(b).
It is also challenging to cover bridge edges with random shortest paths. Differently, our \ltos method could cover not only the paths in dense groups but also the interconnections between diverse dense groups. Overall, \emph{L2SP-based shortest paths have more representative abilities as graph features.}

\subsection{Phase-II: Path Textualization of L2SP} 
To facilitate the Path-LLM's learning of our \ltos structural and semantic features, we utilize a path textual function to represent structural paths into textual sequences, denoted as $\Phi(\cdot)$.
Path textual function $\Phi(\cdot)$ can automatically process and concatenate node text attributes within the path $\mathcal{P}_{short}$ in designed ways to form the text sequence $T$, which can be formulated as $T=\Phi(\mathcal{P}_{short})$. $\Phi(\cdot)$ has different processing ways for two cases: \emph{homogeneous} and \emph{heterogeneous} text-attributed graphs, as shown in Figure~\ref{fig:textual func} and \ref{fig:textual func2}. 

For homogeneous text-attributed graphs, we first conduct data cleaning for node text attributes, removing noisy characters and links. 
Excessively long text attributes and overly complex sentence structures (e.g., abstracts) can weaken the path structure properties in \ltos-based texts. This complexity makes it more difficult for the following Path-LLM to learn the path structure effectively. To address it, we utilize the Positionrank algorithm~\cite{florescu2017positionrank} to extract key phrases from the text attributes as new node attributes $\mathcal{X}'_v$. This simplification retains the primary semantic information while making path structures more accessible for Path-LLM.
We design the \textit{$\langle\text{paper with content...}\rangle$} template for citation networks to assist Path-LLM in distinguishing the text attributes of different nodes.
$\Phi(\cdot)$ integrates templates and different node and edge text attributes into textual paths, as shown in Figure~\ref{fig:textual func}. Formally, $\Phi(\mathcal{P}_{short})=template(
\langle\mathcal{X}'_v|\mathcal{X}_{e}\rangle)$, where $v,e\in \mathcal{P}_{short}$. 
For heterogeneous text-attributed graphs, data cleaning and key phrase selection are also applicable. Given a path $\mathcal{P}_{short}=\langle v_1, v_2,...,v_\ell\rangle$, $\Phi(\cdot)$ concatenates the text attributes of each node on the path as $\Phi(\mathcal{P}_{short})=\langle\mathcal{X}'_{v_1}| \mathcal{X}_{e_{1,2}} |\mathcal{X}'_{v_2}|...\rangle$ to derive \ltos-based texts, as shown in Figure~\ref{fig:textual func2}. 
Thus, we obtain the \ltos-based text by $T=\Phi(\mathcal{P}_{short})$, where $T=\{t_1,...,t_S\}$ and the length is $|T|=S$.

\stitle{Discussion}. Different from WalkLM~\cite{walklm}, our proposed path textual function $\Phi(\cdot)$ 
can handle homogeneous graphs with abundant texts, as shown in Figure~\ref{fig:textual func}. 
For nodes with excessively long text attributes, we conduct data cleaning and key phrase selection, selecting essential tokens as new text attributes of nodes. Moreover, for heterogeneous graphs, WalkLM employs a complicated template unsuitable for LLM path structure learning. 

\begin{figure}[t]
\centering
  \includegraphics[width=0.37\textwidth,height=0.21\textwidth]{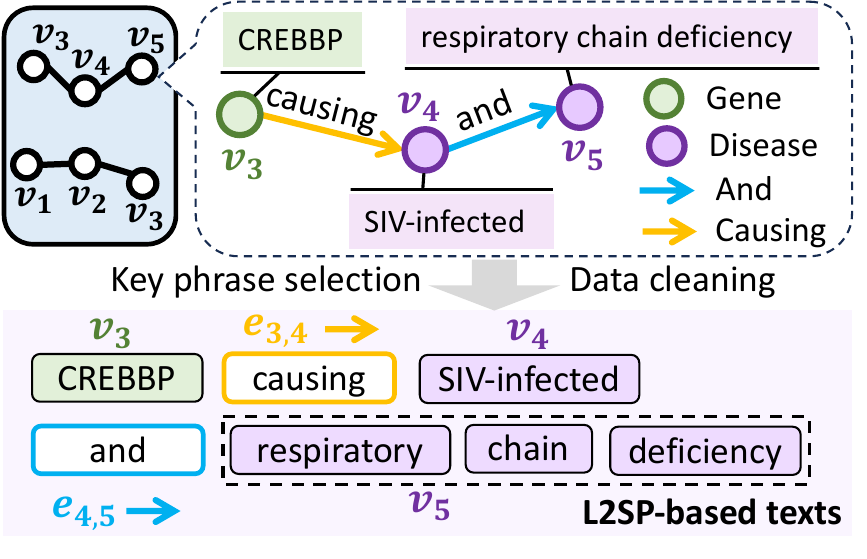}
  \caption{Path textualization of heterogeneous graphs on biomedical networks with four types of nodes: genes, diseases, chemicals and species~\cite{pubmed2}.}
  \label{fig:textual func2}
\end{figure}

\begin{figure*}[t]
    \centering
\includegraphics[width=0.8\textwidth, height=0.26\textwidth]{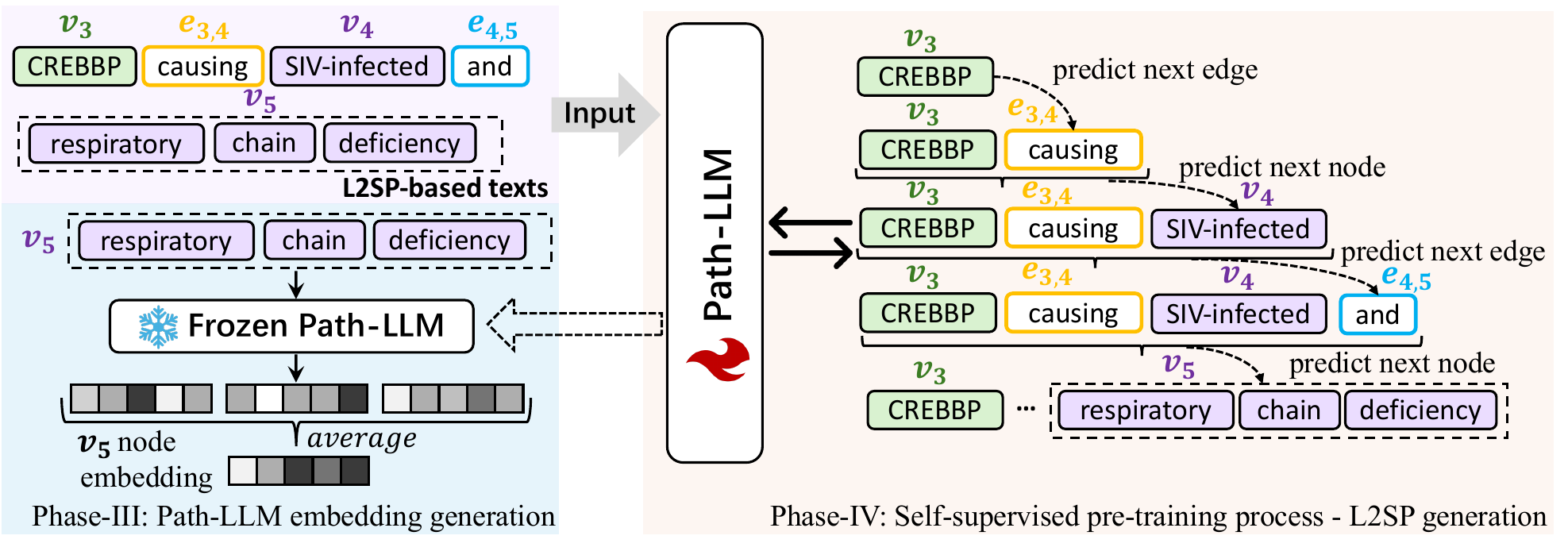}
  \caption{Path-LLM self-supervised pre-training and embedding generation.}
  \label{fig:training}
\end{figure*}

\subsection{Complexity Analysis} 
We analyze the time complexities of L2SP selection in Phase-I and path textualization in Phase-II, respectively. For a graph $G(V, E)$, the size of nodes and edges are $n =|V|$ and $m=|E|$, respectively. W.L.O.G., we assume that $n-1\leq m$ for a connected graph $G$, i.e., $O(n) \subseteq O(m)$.

\begin{theorem}
\label{theorem: l2sp}
    Algorithm~\ref{algo:l2sp} of L2SP selection in Phase-I takes $O(bm)$ time and $O(m+bL_{max})$ space, where $b$ is the number of sampled nodes and $L_{max}$ is the maximum length of paths.  
\end{theorem}

\begin{proof}
First, in the path selection process, we randomly select $b$ nodes. For each sampled node, 
we conduct BFS on $G$, 
taking $O(b(n + m)) \subseteq O(bm)$ time. It then randomly selects one node in a candidate set $\mathcal{T}$, and then conducts BFS again between a source node and a target node, which is $O(n+m) \subseteq O(m) $. Following this, we convert long shortest paths into \ltos‐based shortest paths 
using $O(b\Sigma|\mathcal{P}_{short}|) \subseteq O(bm)$ time, where $\mathcal{P}_{short} \in \mathbb{P}_{short}$.
Therefore, 
the overall time complexity of 
Algorithm~\ref{algo:l2sp}
is $O(bm)$. 

Next, we analyze the space complexity. Algorithm~\ref{algo:l2sp} takes $O(m+n)\subseteq O(m)$ to store $G$ and apply BFS.
It then takes $O(bL_{max})$ space to save the long shortest paths, where $L_{max}$ is the maximum length of the long shortest 
 paths. Thus, the space complexity of Algorithm~\ref{algo:l2sp} is $O(m+bL_{max})$.
\end{proof}

\begin{theorem}
\label{theorem:textualization}
Path textualization in Phase-II takes $O(\sum_{v\in\mathbb{P}} Pos(At\text{-}$ $tr(v)))$ time, where $Pos(\cdot)$ denotes the time complexity of PositionRank algorithm~\cite{florescu2017positionrank} and $Attr(v)$ denotes the number of characters in the text-attribute of $v$.
\end{theorem}

\begin{proof}
In the path textualization phase, we first clean the text attributes of each node $v\in \mathbb{P}_{short}$, leading to a time complexity of $O(\sum_{v\in \mathbb{P}} Attr(v))$, where $Attr(v)$ denotes the number of characters in the text attribute of $v$. For key phrase selection, we employ the PositionRank algorithm~\cite{florescu2017positionrank} to extract key phrases from each cleaned text attribute, incurring a complexity of $O(\sum_{v\in \mathbb{P}} Pos(Token(v)))$, where $Token(v)$ denotes the token number of cleaned text attributes and $Pos(Token(v)))$ denotes the time complexity of PositionRank for one node $v$. 
In the worst case, every token only contains one character, i.e., $O(Token(v))\subseteq O(Attr(v))$. Thus, we can obtain that $O(\sum_{v\in \mathbb{P}} Pos(Token(v))) \subseteq O(\sum_{v\in\mathbb{P}} Pos(Attr(v)))$. Furthermore, based on the PositionRank algorithm, $O(\sum_{v\in \mathbb{P}} Attr(v)) \subseteq O(\sum_{v\in \mathbb{P}} Pos(Attr(v)))$. 
Therefore, the overall time complexity of path textualization is $O(\sum_{v\in\mathbb{P}} Pos(Attr(v)))$.
\end{proof}




\section{Path-LLM Embedding Learning from \ltos based Graph Features}

\label{sec:training and embedding}
This section introduces the Path-LLM self-supervised learning process, designed to learn both inter-group connections and intra-group structures from L2SP-based texts. The trained Path-LLM derives graph embeddings based on node text attributes and learned graph structure features, which are then applied to subsequent downstream tasks (node classification, edge validation, keyword search). The detailed whole process is shown in Figure~\ref{fig:training}.
\subsection{Phase-III: \ltos-based graph feature learning}
{Based on the powerful semantic representation capabilities of LLM, Path-LLM integrates essential graph features inherent in \ltos into its own deep semantic space.} To achieve this, Path-LLM learns to generate \ltos-based texts constructed above through the self-supervised pre-training process. 
{Specifically, we transform the order of nodes in the \ltos to the order of text tokens. As Path-LLM learns the order of tokens, it also learns the order of nodes and edges within the \ltos.}
Typically, the pre-training task for LLM is to generate the next token based on prefix tokens~\cite{gpt3}, similar to Eq.~\ref{llmprob}. 
In this paper, we align the next token generation with the next node or edge generation.
When Path-LLM generates \ltos-based texts, the pre-training task for Path-LLM is to generate the next node $v_i$ or edge $e_{i,j}$ on the corresponding \ltos based on prefix nodes and edges $\{e_{<i}, v_{<i}\}$, as shown in Figure~\ref{fig:training}. This process is then iteratively repeated until the whole \ltos is generated.

Formally, the \ltos-based text can be denoted as $T=\Phi(\mathcal{P}_{short})=\langle...|\mathcal{X}'_{v_i}|\mathcal{X}_{e_{i,j}}|\mathcal{X}'_{v_j}|...\rangle$, 
where $1\leq i,j\leq \ell$. $\mathcal{X}'_{v_i}$ is the processed new text attribute of the node $v_i$, $\mathcal{X}_{e_{i,j}}$ is the text attribute of edge $e_{i,j}$
and $\ell$ denotes the length of $\mathcal{P}_{short}$.  
Based on Eq.~\ref{llmprob}, the probability of Path-LLM predicting the next node $v_i$ can be formulated as $p(\mathcal{X}'_{v_i}|\langle\mathcal{X}'_{v_{<i}},\mathcal{X}_{e_{<i}}\rangle)$ and the probability of predicting the next edge $e_{i,j}$ is $p(\mathcal{X}_{e_{i,j}}|\langle\mathcal{X}'_{v_{\leq i}},\mathcal{X}_{e_{<i}}\rangle)$.
$\langle\mathcal{X}'_{v_{<i}},\mathcal{X}_{e_{<i}}\rangle$ symbolizes the token sequence composed of the nodes and edges preceding $v_i$.
Thus, the probability of generating the whole shortest path $\mathcal{P}_{short}$ is 
\begin{equation}
\label{eq:sp}
     \prod\limits_{i=1}^\ell p(\mathcal{X}'_{v_i}|\langle\mathcal{X}'_{v_{<i}},\mathcal{X}_{e_{<i}}\rangle)p(\mathcal{X}_{e_{i,i+1}}|\langle \mathcal{X}'_{v_{\leq i}},\mathcal{X}_{e_{<i}}\rangle),
\end{equation}
where $p(\mathcal{X}_{e_{i,i+1}}|\langle\mathcal{X}'_{v_{\leq i}},\mathcal{X}_{e_{<i}}\rangle)=1$ if $i=\ell$. 
The probability of generating $\mathcal{P}_{short}$ is consistent with the probability of generating tokens, which is $p(t) = \prod\limits_{i=1}^S p(t_i|t_{<i})$, as proved in Theorem~\ref{theorem-sp}.
Thus, Path-LLM learns the path structure through self-supervised learning, thereby understanding the graph structure.
Consequently, when Path-LLM generates unified graph embeddings, it can integrate the learned graph structure with its inherent deep semantic representation capabilities.

We use the cross-entropy loss function~\cite{zhang2018generalized} during the self-supervised training process of Path-LLM. The \ltos-based text can also be denoted as $T=\{ t_1,...t_S\}$, where $S$ denotes the length of $T$. Path-LLM starts generating from the first token $t_1$ and calculates the loss for generating the next token $t_j$ based on prefix tokens $t_{<j}$. Therefore, the loss for generating one \ltos-based text is:
\begin{equation}
\label{eq:loss}
   \mathcal{L}_\Theta= -\frac{1}{S}\sum\limits_{j=1}^{S} \sum\limits_{k=1}^{|\mathcal{D}|} y_k\log \frac{\exp(P_{t_k^*}(t_j|t_{<j}))}{\sum_{t_k^*\in \mathcal{D}} \exp(P_{t_k^*}(t_j|t_{<j}))}, 
\end{equation}
where $\Theta$ is the learnable parameters of Path-LLM, $j$ is the position in $T$ and $k$ is the position in the LLM vocabulary $\mathcal{D}$. $P_{t_k^*}$ represents the probability that the $k$-th token $t_k^*$ in $\mathcal{D}$ is the next token. $|\mathcal{D}|$ denotes the size of the LLM vocabulary.
$y_k$ is the label of whether $t_k^*$ in $\mathcal{D}$ is the token at position $j$.
For the whole dataset composed of $N$ \ltos-based texts, the final loss function is:
 $$\mathcal{L}_{pretrain}=  \frac{1}{N}\sum\limits_{N} \mathcal{L}_\Theta,$$   

\stitle{Rationales of Pre-training in Phase-III.}
During the Path-LLM pre-training process, \ltos-based text generation can be regarded as \ltos generation. 
Formally, 
tokens of node and edge attributes can be denoted as $\{t_j^{v_i}\}$ and $\{t_j^{e_{i,i+1}}\}$ separately,
where $1\leq j\leq S$ and $1\leq i\leq \ell$. $j$ represents the token index in $T$, while $i$ represents the node index in the corresponding path. $S$ is the length of the text, and $\ell$ is the length of the corresponding \ltos. 
Tokens associated to node $v_i$ can be denoted as $\{t_{j^*\leq j\leq \Tilde{j}}^{v_i}\}$, where $t_{j^*}^{v_i}$ is the first token in $\mathcal{X}'_{v_i}$ and $t_{\Tilde{j}}^{v_i}$ is the last token in $\mathcal{X}'_{v_i}$. The probability of Path-LLM generating tokens of node $v_i$~\cite{yadkori2024believe}\cite{requeima2024llm} is $\prod\limits_{j=j^*}^{\Tilde{j}} p(t_{j}^{v_i}|t_{<j}^{v_{\leq i}},t_{<j}^{e_{<i}})$, 
where $t_{<j}^{v_{\leq i}}$ represents a token list satisfying the token index being less than $j$ and the node index being less than or equal to $i$, while $t_{<j}^{e_{<i}}$ represents the tokens associated with the edges preceding $v_i$. Similarly, tokens associated to edge $e_{i,i+1}$ can be denoted as $\{t_{j'\leq j\leq j''}^{e_{i,i+1}}\}$ and the probability of generating tokens of edge $e_{i,i+1}$ is $\prod\limits_{j=j'}^{j''} p(t_{j}^{e_{i,i+1}}|t_{<j}^{v_{\leq i}},t_{<j}^{e_{<i}})$. 
Then, we can formulate the probability of generating the \ltos-based text $T$ as $\prod\limits_{i=1}^{\ell}\left[ \prod\limits_{j=j^*}^{\Tilde{j}} p(t_{j}^{v_i}|t_{<j}^{v_{\leq i}},t_{<j}^{e_{<i}}) \prod\limits_{j=j'}^{j''} p(t_{j}^{e_{i,i+1}}|t_{<j}^{v_{\leq i}},t_{<j}^{e_{<i}})\right]$, which equals $p(t) = \prod\limits_{i=1}^S p(t_i|t_{<i})$. 
While, the probability of the corresponding \ltos generation can be formulated as Eq.~\ref{eq:sp}.

\begin{lemma}
\label{lemma-node}
    During the Path-LLM self-supervised pre-training process, the generation of tokens associated with nodes in \ltos-based texts can be regarded as the generation of nodes in \ltoss. Formally, $$ \prod\limits_{j=j^*}^{\Tilde{j}} p(t_{j}^{v_i}|t_{<j}^{v_{\leq i}},t_{<j}^{e_{<i}}) = p(\mathcal{X}'_{v_i}|\langle\mathcal{X}'_{v_{<i}},\mathcal{X}_{e_{<i}}\rangle).$$
\end{lemma}
\begin{proof}
\label{node-proof}
    Based on the chain rule of conditional probability \cite{chainrule}, we can derive that
    \begin{equation}
        \prod\limits_{j=j^*}^{\Tilde{j}} p(t_{j}^{v_i}|t_{<j}^{v_{\leq i}},t_{<j}^{e_{<i}})= p(\{t_{j^*\leq j\leq \Tilde{j}}^{v_i}\}|t_{< j^*}^{v_{<i}},t_{<j^*}^{e_{<i}}). 
    \end{equation}

    Due to $\mathcal{X}'_{v_i}=\{t_{j^*\leq j\leq \Tilde{j}}^{v_i}\}$ and $\langle\mathcal{X}'_{v_{<i}},\mathcal{X}_{e_{<i}}\rangle=\{t_{< j^*}^{v_{<i}},t_{<j^*}^{e_{<i}}\}$, we can know that:
    $$
        \prod\limits_{j=j^*}^{\Tilde{j}} p(t_{j}^{v_i}|t_{<j}^{v_{\leq i}},t_{<j}^{e_{<i}}) =p(\mathcal{X}'_{v_i}|\langle\mathcal{X}'_{v_{<i}},\mathcal{X}_{e_{<i}}\rangle).
    $$
\end{proof} 

\begin{lemma}
\label{lemma-edge}
    During the Path-LLM self-supervised pre-training process, the generation of tokens associated with edges in \ltos-based texts can be regarded as the generation of edges in \ltoss. Formally, $$\prod\limits_{j=j'}^{j''} p(t_{j}^{e_{i,i+1}}|t_{<j}^{v_{\leq i}},t_{<j}^{e_{<i}}) =p(\mathcal{X}_{e_{i,i+1}}|\langle\mathcal{X}'_{v_{\leq i}},\mathcal{X}_{e_{<i}}\rangle).$$
\end{lemma}
\begin{proof}
    Similar to the proof of Lemma~\ref{lemma-node}.
\end{proof}

\begin{theorem}
\label{theorem-sp}
During the Path-LLM self-supervised pre-training process, \ltos-based text generation can be regarded as the \ltos generation.
Formally, 
\begin{equation}
\begin{split}
    &\prod\limits_{i=1}^{\ell}\left[ \prod\limits_{j=j^*}^{\Tilde{j}} p(t_{j}^{v_i}|t_{<j}^{v_{\leq i}},t_{<j}^{e_{<i}}) \prod\limits_{j=j'}^{j''} p(t_{j}^{e_{i,i+1}}|t_{<j}^{v_{\leq i}},t_{<j}^{e_{<i}})\right]\\=&\prod\limits_{i=1}^\ell p(\mathcal{X}'_{v_i}|\langle\mathcal{X}'_{v_{<i}},\mathcal{X}_{e_{<i}}\rangle)p(\mathcal{X}_{e_{i,i+1}}|\langle\mathcal{X}'_{v_{\leq i}},\mathcal{X}_{e_{<i}}\rangle).
\end{split}
\end{equation}

\end{theorem}
\begin{proof}
    By Lemmas~\ref{lemma-node} and \ref{lemma-edge}, we can prove Theorem~\ref{theorem-sp}.
\end{proof}

\subsection{Phase-IV: Path-LLM Embedding Generation}

The node text attribute $\mathcal{X}_{v_i}$ can be a few words, a sentence, or a paragraph, all of which can be transformed into Path-LLM embeddings. Specifically, after cleaning $\mathcal{X}_{v_i}$ and selecting key phrases as new text attributes $\mathcal{X}'_{v_i}$, we input $\mathcal{X}'_{v_i}$ into the frozen Path-LLM and then extract all token embeddings $\{\varepsilon_i\}_{i=1}^{r}$ associated with the same node $v_i$, where $r$ is the number of token embeddings from the same node. 
We extract embeddings from the last layer of the Path-LLM, 
where 
$\varepsilon_i \in \mathbb{R}^d$ and $d$ is the embedding dimension. 
To obtain the Path-LLM embedding for node $v_i$, the token embeddings derived from the text attribute of $v_i$ are averaged,i.e., for a node $v_i$, the Path-LLM embedding $\xi_{v_i}\in \mathbb{R}^d$ is calculated as $\xi_{v_i} = \frac{1}{r}\sum\limits_{i=1}^{r} \varepsilon_i$.
The embedding $\xi_{v_i}$ integrates both the graph structure learned by Path-LLM and its inherent semantic understanding capabilities.

\begin{table}[t]
\small
   \centering
   \renewcommand\arraystretch{1.2}
   \setlength{\tabcolsep}{4pt}
   \caption{Complexity analysis of WalkLM and Path-LLM in two stages: fine-tuning and embedding deriving. $B_{LM}(\cdot)$ denotes the time complexity taken by a LM in WalkLM. $|N|$ denotes the number of random walks. $l_{avg}$ denotes their average length. Path-LLM significantly faster than WalkLM in two reasons: $|\mathbb{P}_{short}|<<|N|$ and $\hat{S}<<l_{avg}$.}
   \label{tab:walklmsetting}
   \begin{tabular}{l|c|c}\hline
   Method  &   Fine-tuning &  Embedding-deriving \\
\hline
    WalkLM &  $O(B_{LM}(I,|N|, l_{avg}, d)) $ & $O(B_{LM}(n,l_{avg},d))$\\  \hline
    Path-LLM &  $O(B_{LLM}(I,|\mathbb{P}_{short}|,\hat{S},d))$ & $O(B_{LLM}(n,\hat{S},d))$\\

  \hline
    \end{tabular}
\end{table}

\stitle{Complexity Analysis}.{
We analyze the time complexity of Phase-III and Phase-IV. 
Due to the complex structure of the neural network-based learning model, we assume that a commonly used LLM is a black-box model taking a function of time complexity, denoted by $B_{LLM}(\cdot)$. 
Thus, our Path-LLM takes $O(B_{LLM}(I,|\mathbb{P}_{short}|,\hat{S},d))$, where $I$ is the number of iterations for Path-LLM training, $|\mathbb{P}_{short}|$ is the number of training samples, which equals the total number of \ltoss, $\hat{S}$ is the average length of \ltos-based texts and $d$ is the dimension of embedding. 
When parameters $I, |\mathbb{P}_{short}|, \hat{S}, d$ increase, the complexity $O(B_{LLM}(I,|\mathbb{P}_{short}|,\hat{S},d))$ also increases. Notably, due to the attention mechanism of texts generally used in LLMs, $O(B_{LLM}(I,|\mathbb{P}_{short}|,\hat{S},d))$ is positively correlated with $\hat{S}^2$.
For Path-LLM embedding generation, the time complexity of Path-LLM deriving embeddings is related to the node size $|V|=n$, the length of node text-attributes $r$, where $r<\hat{S}$ and the dimension of embedding $d$. After deriving embeddings, the time complexity of the average operation is $O(nrd)$. Therefore, the total time complexity of Phase-III and Phase-IV is $O(B_{LLM}(I,|\mathbb{P}_{short}|,n,\hat{S},d)+nrd)$, where the complexity is positively correlated with all input parameters.}


\subsection{Path-LLM Algorithm} We propose Path-LLM to derive unified graph embeddings from a given text-attributed graph, which is detailed in Algorithm~\ref{algo:framework}. 
We first sample L2SP-based shortest paths in Phase-I (line 1). Then, in Phase II (lines 2-10), we perform path textualization to convert the L2SP-based shortest paths into graph-attributed texts. The resulting graph-attributed texts, referred to as L2SP-based texts, serve as the training corpus for Path-LLM. In Phase III, we train Path-LLM utilizing a self-supervised learning approach (lines 11-19). $\mathcal{B}$ in line 13 denotes one training batch. The number of iterations can be calculated as $I_{max} = \frac{I\cdot |\mathbb{P}_{short}|}{|\mathcal{B}|}$. $|\mathcal{B}|$ in line 17 represents the batch size. We demonstrate in Theorem~\ref{theorem-sp} that generating the subsequent token within L2SP-based texts (line 15) can also be interpreted as generating the next node or edge in L2SP-based shortest paths. Finally, in Phase IV (lines 20-23), Path-LLM extracts unified graph embeddings, with $f(\mathcal{X}'_v)$ denoting the embeddings produced by inputting $\mathcal{X}'_v$ into the Path-LLM function $f(\cdot)$.

\begin{theorem}
\label{theorem:framework}
{Path-LLM in Algorithm~\ref{algo:framework} takes $O(bm+\sum_{v\in\mathbb{P}} Pos(At\text{-}$ $tr(v))+B_{LLM}(I,|\mathbb{P}_{short}|,n,\hat{S},d)+nrd)$ time.}
\end{theorem}

\begin{proof}
We conclude the complexity by integrating the complexity results by Theorems~\ref{theorem: l2sp}, ~\ref{theorem:textualization}, and the  analysis of Phase-III and Phase-IV in Section~\ref{sec:training and embedding}.
\end{proof}

\begin{algorithm}[t]
\small
  \caption{Path-LLM Framework}
  \label{algo:framework}
  \begin{algorithmic}[1] 
    \Require A text-attributed graph $G=(V, E, \mathcal{X}_v,\mathcal{X}_e)$.
    \Ensure Path-LLM graph embeddings $\xi_v$.
    \State Phase-I: Generate a series of shortest paths $\mathbb{P}_{short}$ on $G$ by the L2SP selection in Algorithm~\ref{algo:l2sp};
    \State Phase-II: Transform $\mathbb{P}_{short}$ into graph-attributed texts
    \For{$\mathcal{P}_{short} \in \mathbb{P}_{short}$}
        \For{$v \in \mathcal{P}_{short}$}
            \State Remove noisy characters and links from $\mathcal{X}_v$;
            \State Select key phrases from $\mathcal{X}_v$ as new $\mathcal{X}'_v$;
        \EndFor
        \If{$G$ is a homogeneous graph}
            \State $T=\Phi(\mathcal{P}_{short})=template(\langle\mathcal{X}'_v|\mathcal{X}_e\rangle)$ by Fig.~\ref{fig:textual func};
        \EndIf
        \If{$G$ is a heterogeneous graph}
            \State $T=\Phi(\mathcal{P}_{short})=\langle...|\mathcal{X}'_v|\mathcal{X}_e|...\rangle$ by Fig.~\ref{fig:textual func2};
        \EndIf
    \EndFor
    \State Phase-III: Path-LLM self-supervised pre-training
    \For{$1\leq iteration \leq I_{max}$}
        \For{$T \in$ $\mathcal{B}$}
            \For{$t \in T$}
                \State Generate next token $t_j \in T$;
            \EndFor
            \State Compute loss $ \mathcal{L}_\Theta$ as Eq.~\ref{eq:loss};
        \EndFor
        \State Compute $ \mathcal{L}_{batch}=\frac{1}{|\mathcal{B}|}\sum\limits_{|\mathcal{B}|}\mathcal{L}_\Theta$;
        \State Differentiate the loss function $\mathcal{L}_{batch}$ and use the chain rule to backpropagate the gradient to calculate gradient of the parameters $g$;
        \State Update Path-LLM parameters $w \leftarrow w-\eta g$, where $\eta$ is the adaptive learning rate that adjusts the parameter updates;
    \EndFor
    \State Phase-IV: Path-LLM embedding generation
    \For{$v_i \in V$}
        \State $\{ \varepsilon_{i}\}^{r}_{i=1}\leftarrow f(\mathcal{X}'_{v})$ from frozen Path-LLM;
        \State $\xi_{v_i} = \frac{1}{r}\sum\limits_{i=1}^{r} \varepsilon_i$;
        \State $\xi_{v}\leftarrow \xi_{v}\cup \{\xi_{v_i}\}$;
    \EndFor
    \State \Return{$\xi_{v}$}; 
  \end{algorithmic}
\end{algorithm}

\section{PATH-LLM BASED KEYWORD SEARCH}
\label{new-task}
 
To illustrate the 
usefulness of Path-LLM, 
we investigate the application of our learned Path-LLM embeddings to help tackle one typical 
graph analytics 
task of keyword search. 

\stitle{Keyword search}. Keyword search is a classical graph query processing task widely applied in databases and recommendation systems~\cite{wang2010survey, li2008ease}, particularly retrieval-augmented generation~\cite{he2024g}. 
Given a graph $G$ associated with node keywords and a set of query keywords $Q$, 
the goal of keyword search is to find the best subgraph $H$ of $G$ such that $H$ covers all keywords $Q$ and has the smallest edge weights. 
Following~\cite{keywordsearch}, we formulate the problem as follows. Given a weighted graph $G=(V, E, W)$ 
a set of query keywords $Q=\{q_1,...,q_m\}$, 
the problem returns a tree structure $T\subseteq G$, 
such that
\begin{enumerate}
    \item $V_T$ covers all keywords in $Q$;
    \item the edge weight $f(T)$ is minimized among all feasible choices, where $f(T)=\sum_{e_{i,j}\in E_T} w_{i,j}$.
\end{enumerate}
However, in many real-world applications of keyword search, it strictly follows the original structure of given graphs but neglects the semantic connections of node weights, which limits the discovery of real close subgraph patterns w.r.t. the given keywords. Even worse, in some cases, the edge weights are not given in advance for the querying graph. In this paper, we \emph{leverage the Path-LLM embedding to reconstruct a new weighted graph $G^*(V, E^*, W^*)$} and \emph{search for an integrated keyword answer with close connections in terms of both topology structure and node semantics}. 

\begin{table*}[t]
\small
\centering
\setlength{\tabcolsep}{8pt}
  \caption{Statistics of tested datasets.}
  \label{tab:dataset}
  \begin{tabular}{ccccccc}
    \toprule
     & \multicolumn{3}{c}{Statistics of graphs} & \multicolumn{3}{c}{The number of training paths} \\ 
     \cmidrule(lr){2-4}  \cmidrule(lr){5-7}
    Dataset & \#Nodes & \#Edges & Graph type  & 
    \#RW-paths~\cite{walklm} 
    &  \#\ltos-paths &  The ratio of saving paths \\
    \hline
     PubMed~\cite{pubmed2}  &  63,109  & 244,986 &   \emph{heterogeneous}  &  295,512  & 19,670  &  \textbf{93.35\%}  \\ 
     Freebase~\cite{bollacker2008freebase} & 180,098  & 1,057,688   &   \emph{heterogeneous} &  300,640    & 28,041   &  \textbf{90.67\%}  \\
     \hline
    Cora~\cite{cora}  &  2,708  & 5,429 & \emph{homogeneous}   &  100,000  & 12,932   & \textbf{87.07\%}  \\
    Citeseer~\cite{citeseer}  & 3,312  & 8,554  &  \emph{homogeneous}  & 100,000  & 11,504  &  \textbf{88.50\%} \\
    Facebook~\cite{leskovec2012learning}  &  4,039 & 88,234 & \emph{homogeneous} & 124,881 & 9,010 & \textbf{92.78\% }  \\
    OGB-ARXIV~\cite{arxiv}  &  169,343  &  1,166,243 &  \emph{homogeneous} &  350,000  &  20,523   &  \textbf{94.14\%}   \\
  \bottomrule
\end{tabular}
\end{table*}

\begin{table*}[t]
\small
   \centering
   \renewcommand\arraystretch{1.02}
   \setlength{\tabcolsep}{3pt}
   \caption{Mean results of node classification on six datasets in self-supervised settings: PubMed, Freebase, Cora, ARXIV, Citeseer and Facebook. The best performances are in bold, and the second-best are shaded in gray.}
   \label{tab:node}
   \begin{tabular}{lcccccccccccc}\hline
   Datasets & \multicolumn{2}{c}{PubMed} & \multicolumn{2}{c}{Freebase} & \multicolumn{2}{c}{Cora} & \multicolumn{2}{c}{ARXIV} & \multicolumn{2}{c}{Citeseer} & \multicolumn{2}{c}{Facebook} \\ 

   \cmidrule(lr){2-3} \cmidrule(lr){4-5}  \cmidrule(lr){6-7} \cmidrule(lr){8-9} \cmidrule(lr){10-11} \cmidrule(lr){12-13}
    Metrics & Macro-F1 & Micro-F1 & Macro-F1 & Micro-F1 & Macro-F1 & Micro-F1 & Macro-F1 & Micro-F1 & Macro-F1 & Micro-F1 & Macro-F1 & Micro-F1 \\ \hline

    GCN~\cite{gcn}  & 0.2593  &  0.3570 & 0.0411 & 0.1197   &  0.3628  & 0.4040 & 0.1681 & 0.5199 & 0.4627 & 0.4903 & 0.1542 & 0.2567 \\  
    GraphSage~\cite{graphsage}  & 0.2140 & 0.2430 & 0.0486 & 0.1041 & 0.3684 & 0.4140 & 0.1962 & 0.5581 & 0.4903 & 0.5240 & 0.2641 & 0.36  \\
    GATv2~\cite{gatv2}  & 0.2501 & 0.2870 & 0.0859 & 0.1468 &  0.3706 & 0.4740 & 0.1199 & 0.4443 & 0.4816 & 0.5530 &  0.4354 & 0.2573 \\ \hline
    WalkLM~\cite{walklm} & 0.2871 &  0.4070 & 0.2014 & 0.5437 & 0.4031 & 0.5336  & 0.0872 & 0.3823 & 0.5761 & 0.6616 & 0.3327 & 0.51    \\ 
    Llama 2~\cite{llama2} &  $\colorbox{Gainsboro}{\text{0.7167}}$  &  $\colorbox{Gainsboro}{\text{0.7246}}$ & 0.5141 & 0.7087   & 0.6608  &  0.6946  & 0.3484 & 0.5748 & 0.6739 & 0.7281 & 0.4053 & 0.5366  \\ 
    OFA~\cite{liu2024one} & 0.6520 & 0.6761 & 0.5376 & 0.7178 & 0.6720  & 0.7045 & 0.3199 & 0.5792  & 0.6280   & 0.6987 & 0.4053 & 0.5366 \\
    GraphTranslator~\cite{zhang2024graphtranslator} & 0.4917  & 0.5396 & 0.4417 & 0.6768  & 0.5998   & 0.6503 & 0.3431 & 0.5651 & 0.6445 & 0.7077 & 0.3997 &  $\colorbox{Gainsboro}{\text{0.5533}}$ \\
    GraphGPT~\cite{tang2024graphgpt} & 0.7088  & 0.7202 &  $\colorbox{Gainsboro}{\text{0.5715}}$ &  $\colorbox{Gainsboro}{\text{0.7359}}$  & $\colorbox{Gainsboro}{\text{0.6766}}$   & $\colorbox{Gainsboro}{\text{0.7097}}$ & $\colorbox{Gainsboro}{\text{0.3610}}$ & $\colorbox{Gainsboro}{\text{0.5811}}$ & $\colorbox{Gainsboro}{\text{0.6785}}$ & $\colorbox{Gainsboro}{\text{0.7335}}$ &  $\colorbox{Gainsboro}{\text{0.4166}}$  & 0.5400  \\
     \hline
    \textbf{Path-LLM (Ours)}  & \textbf{0.7595}  & \textbf{0.7674} &\textbf{ 0.6101} & \textbf{0.7419}   & \textbf{0.7524}  & \textbf{0.7773} & \textbf{0.4529} & \textbf{0.6616}  & \textbf{0.7034} & \textbf{0.7545}  & \textbf{0.4409} & \textbf{0.5600} \\
  \hline
    \end{tabular}
\end{table*}

\stitle{
{Graph structure construction}}. 
Given a TAG $G (V, E)$, we construct a new structure of graph $G^*$ based on the Path-LLM node embedding.
Specifically, we first collect all isolated nodes $V$ and then add edges to them as follows.
For an edge $(v_i, v_j)\in E$, we measure its importance by calculating cosine-similarity \cite{cosine} based on their embedding vectors $\xi_i,\xi_j$ derived from Path-LLM, i.e., 
\begin{equation}
        \psi_{i,j} = \frac{\xi_i\cdot\xi_j}{\lVert \xi_i \rVert \lVert \xi_j \rVert },\ \text{where}\  \psi_{i,j}\in [-1,1]. 
\end{equation}
For $\psi_{i,j}<0$, we consider $\xi_i$ and $\xi_j$ to be dissimilar, which is treated as the minimum non-negative edge weight of 0. Thus, we design a mapping function $f^*$ as:
\begin{equation}
\label{edge-importance}
    f^*(e_{i,j}) = \left\{ \begin{array}{rcl}
    0 & \mbox{for}
    & \psi_{i,j} \leq 0 \\ \psi_{i,j} & \mbox{for} & \psi_{i,j} > 0 .
    \end{array}\right. 
\end{equation}
Only for a positive weight $\psi_{i,j}$, we add an edge connection between $v_i$ and $v_j$. Thus, the new edge set is $E^*=\{e_{i,j}: v_i, v_j\in V, f^*(e_{i,j})>0\}$. 
Therefore, we obtain the topology structure $V$ and $E^*$ of new TAG $G^*$ without weights.

\stitle{
{Edge importance weights assignment in 
$G^*$}}. 
{To find the most important subgraph $T$ satisfying the requirements of both with minimum edges and semantically strongest related, covering all keywords in $Q$, we assign edge weights based on edge importance values. Considering one subgraph $T$, the importance of $T$ is measured as the product of all edge importance values in $T$, i.e., $\prod_{e_{i,j}\in E_T} f^*(e_{i,j})$. Next, we convert the edge importance to the edge weight as $w^*_{i,j} = -\log{f^*(e_{i,j})}$, thereby transforming our keyword search task into the traditional keyword search task as follows, }
\begin{equation}
        \nonumber
        \min\limits_{E_T\in E^*}\sum w^*_{i,j} = \min\limits_{E_T\in E^*}\sum \text{-}\log f^*(e_{i,j}) \Leftrightarrow \max\limits_{E_T\in E^*}\prod f^*(e_{i,j}).
\end{equation}
{The goal of finding traditional Steiner Tree is $\min \sum_{e_{i,j}\in E_T} w^*_{i,j}$, the same as finding the most important subtree $T$, $\max \prod_{e_{i,j}\in E_T} f^*(e_{i,j})$.}
As a result, we finally obtain a new weighted TAG $G^* (V, E^*, W^*)$. Across this TAG, we can search the subtree that satisfies both the fewest edges and semantically strongest related.

\stitle{Our solution of keyword search over $G^*$}. We first consider one simple case of keyword search with $|Q|=2$ for two query nodes $v_i$, $v_j$. We use Dijkstra's algorithm~\cite{dijkstra} to search for the shortest path between $v_i$ and $v_j$ based on the new weight $W^*$. For a general keyword search with $|Q|\geq 3$, 
we adopt a $2$-approximation greedy algorithm~\cite{mehlhorn1988faster} for finding a Steiner Tree to cover all keywords $Q$ to tackle this NP-hard problem.

\section{EXPERIMENTS}
\label{experiment}

\stitle{Datasets.}
We conduct evaluations on six diverse datasets in Table~\ref{tab:dataset}, encompassing graphs of varying scales and types: PubMed~\cite{pubmed2}, Freebase~\cite{bollacker2008freebase}, Cora~\cite{cora}, Citeseer~\cite{citeseer}, Facebook~\cite{leskovec2012learning}, and ARXIV~\cite{arxiv}. 
PubMed is a biomedical network containing four types of nodes: genes, diseases, chemicals, and species~\cite{pubmed2}. Freebase is a multi-domain knowledge graph with eight types of nodes, ranging from books to business~\cite{bollacker2008freebase}. Facebook is a social network with textual user attributes~\cite{leskovec2012learning}.
The other three datasets are citation networks with text attributes like title, abstract, and so on. 
Raw text data of Cora and Citeseer are collected from \cite{datasource}.  
Table~\ref{tab:dataset} also reports the number of training paths for WalkLM and Path-LLM, respectively. Our Path‐LLM, as an effective learning model, \emph{uses significantly fewer L2SP‐based shortest paths} compared to the number of random‐walk‐based paths in WalkLM~\cite{walklm}, \emph{saving the number of training paths by an average of 91.09\% across six datasets.}

\begin{table*}[t]
\small
   \centering
   \renewcommand\arraystretch{1.02}
   \setlength{\tabcolsep}{4pt}
   \caption{The results of edge validation on six datasets in self-supervised settings: PubMed, Freebase, Cora, ARXIV, Citeseer and Facebook. The best performances are in bold, and the second-best are shaded in gray. Corresponding std are provided in Appendix  .}
   \label{tab:link}
   \begin{tabular}{lcccccccccccc}\hline

   Datasets & \multicolumn{2}{c}{PubMed} & \multicolumn{2}{c}{Freebase} & \multicolumn{2}{c}{Cora} & \multicolumn{2}{c}{ARXIV} & \multicolumn{2}{c}{Citeseer} & \multicolumn{2}{c}{Facebook} \\ 
   \cmidrule(lr){2-3} \cmidrule(lr){4-5}  \cmidrule(lr){6-7} \cmidrule(lr){8-9} \cmidrule(lr){10-11} \cmidrule(lr){12-13}
    Metrics & AUC & Accuracy & AUC & Accuracy & AUC & Accuracy & AUC & Accuracy & AUC & Accuracy  & AUC & Accuracy \\ \hline
    GCN~\cite{gcn}  & 0.5155 & 0.5426 &  0.5012 & 0.0144  & 0.6680 & 0.7018 & 0.5216 & 0.4415 & 0.6392& 0.6052 & 0.4354 & 0.2573 \\  
    GraphSage~\cite{graphsage}  & 0.5133 & 0.5211 & 0.5426 & 0.6519 & 0.7445 & 0.4052 & 0.5110 & 0.3120 & 0.7822 & 0.5227 & 0.5157 &  0.3900 \\
    GATv2~\cite{gatv2}  & 0.5204 & 0.5011 & 0.6065 & 0.5390  & 0.5143   & 0.4488 & 0.2550 & 0.6120 & 0.7488 & 0.5986 &  0.2117 & 0.39  \\ \hline
    WalkLM~\cite{walklm} & 0.5962 & 0.5684 & 0.7039  & 0.6519  & 0.8581 & 0.7746 & 0.8799 & 0.7923 & 0.9149 & 0.8424 &  0.5665 & 0.5251    \\ 
    Llama 2~\cite{llama2} & $\colorbox{Gainsboro}{\text{0.7144}}$ &  $\colorbox{Gainsboro}{\text{0.6665}}$ & 0.8183  & 0.7481 & 0.8568 &  0.7809 & 0.9157 & 0.8379 & 0.9290 & 0.8550 & 0.6454 & 0.5891 \\ 
    OFA~\cite{liu2024one} & 0.6061 & 0.5719 & 0.8203 & 0.7517  & 0.8473	& 0.7585 & 0.9091 & 0.8291  & 0.8979 & 0.8125  & 0.5739 & 0.5380 \\
    GraphTranslator~\cite{zhang2024graphtranslator} & 0.5939 & 0.5618 & 0.7842  & 7157   & 0.8473 & 0.7585 & 0.9091 & 0.8291 & 0.8979 & 0.8124 & 0.6049 & 0.5891  \\
    GraphGPT~\cite{tang2024graphgpt} & 0.7134  & 0.6631 & $\colorbox{Gainsboro}{\text{0.8381}}$ & $\colorbox{Gainsboro}{\text{0.7683}}$  & $\colorbox{Gainsboro}{\text{0.8679}}$   & $\colorbox{Gainsboro}{\text{0.7857}}$ & $\colorbox{Gainsboro}{\text{0.9221}}$ & $\colorbox{Gainsboro}{\text{0.8430}}$ &  $\colorbox{Gainsboro}{\text{0.9329}}$ & $\colorbox{Gainsboro}{\text{0.8570}}$ & $\colorbox{Gainsboro}{\text{0.6921}}$ & $\colorbox{Gainsboro}{\text{0.6357}}$  \\
     \hline
    \textbf{Path-LLM (Ours)}  & \textbf{0.7497}  &  \textbf{0.7111} &  \textbf{0.8456} & \textbf{0.7794}  & \textbf{0.9244}  &  \textbf{0.8476} & \textbf{0.9655} & \textbf{0.9060} & \textbf{0.9627} & \textbf{0.8966} & \textbf{0.7019} & \textbf{0.6423} \\
  \hline
    \end{tabular}
\end{table*}

\stitle{Competitive methods.}
We compare our 
Path-LLM with the SOTA GNN-, LM- and LLM-based methods as follows. 
\squishlisttight
\item \textbf{Three classical graph learning models}: 
We test three methods, GCN~\cite{gcn}, GraphSage~\cite{graphsage}, and GATv2~\cite{gatv2}.
The GATv2~\cite{gatv2} proposes a dynamic graph attention variant and captures more expressive graph structures.
\item \textbf{WalkLM}~\cite{walklm}: WalkLM is the state-of-the-art LM-based method, integrating random walks and RoBERTa~\cite{liu2019roberta} for unified graph representation learning.
\item \textbf{Llama 2-7B}~\cite{llama2}: Llama 2-7B is the widely used large language model for graph embeddings with exceptionally powerful semantic representation abilities.
\item \textbf{OFA}~\cite{liu2024one}: OFA uses natural language to describe nodes and then encodes texts to feature vectors as node embeddings.
\item \textbf{GraphTranslator}~\cite{zhang2024graphtranslator}: GraphTranslator proposes the graph-text alignment method to derive node embeddings.
\item \textbf{GraphGPT}~\cite{tang2024graphgpt}: GraphGPT proposes to encode graph structures and utilizes a self-supervised instruction tuning method for Vicuna-7B-v1.5 training, which is the state-of-the-art LLM-based baseline for graph representation.
\squishend

\stitle{Experimental settings.} {We mainly compare eight advanced baselines under the setting of self-supervised graph learning.}
{Most experimental settings follow WalkLM, including train-test data split ratio 8:2, five-fold cross-validation, and the one-layer MLP, except for epoch settings and supplementary models.} 
To assess the effectiveness of different graph embeddings, we use a simple but challenging setting designed to \emph{better differentiate embedding performance while saving cost and time}. 
For node classification, the one-layer MLP training epoch is set to 50, compared to 2000 epochs in WalkLM. For edge validation, the MLP training epoch is set to 100 without any additional models, compared to 2000 epochs in WalkLM with the LMNN as supplementary models.
We choose Llama 2‐7B~\cite{llama2} as the base LLM in our Path-LLM method. All models are optimized using the Adam optimizer~\cite{adam} with an initial learning rate of 2e‐4. The rank and scaling factors of the LoRA adapter~\cite{lora} are set to 16 and 32, respectively. The minimum length of the long shortest path is set to $L=10$ for particular social networks with an average distance of $4$ to $7$. For other parameters, we set $k=5$, $b=1000$ and $\ell=3$. 
All experiments are implemented with an NVIDIA A100 (80G) GPU. 


\stitle{Exp-1: Effectiveness on node classification task.}
Node classification task is to assign labels to nodes in a graph based on node embeddings integrating the attributes of nodes and the relationships between them.
For node classification, we train a separate one-layer MLP classifier based on all unified graph embeddings and evaluate graph embeddings derived from all methods with Macro-F1 (across all labels) and Micro-F1 (across all nodes)~\cite{f1}. Note that for GNNs, we train GNNs through contrastive learning to generate unified graph embeddings and then feed them into MLP for downstream~tasks.
As shown in Table~\ref{tab:node}, 
our proposed Path-LLM showcases substantial performance enhancements over WalkLM and other GNN-based and LLM-based competitors on six diverse datasets.
For PubMed, Path-LLM achieves a remarkable 174.56\% relative improvement in macro-F1 and 104.24\% in micro-F1 over WalkLM. 
Moreover, Path-LLM outperforms WalkLM by achieving a 72.58\% relative improvement in macro-F1 and 36.82\% in micro-F1 on Cora. Even on large-scale graphs, Path-LLM outperforms WalkLM by achieving a 419.38\% relative improvement in macro-F1 and 73.06\% in micro-F1 on ARXIV.
In addition, it demonstrates that the shortest path structure can significantly enhance the effectiveness of Path-LLM, with an average 4.66\% gain of the pure LLM method Llama 2~\cite{llama2} on PubMed and an average 12.89\% gain of  Llama 2~\cite{llama2} on Cora, verifying that the structural information learned by Path-LLM is beneficial for node classification. 
Furthermore, it also demonstrates that our self-supervised approach can learn graph structure features better for unified graph embeddings compared to the existing state-of-the-art LLM self-supervised method. 
Compared to GraphGPT~\cite{tang2024graphgpt}, Path-LLM outperforms GraphGPT across all six datasets. On large-scale graphs, Path-LLM achieves relative gains of 25.46\% in macro-F1 on ARXIV and 6.75\% on Freebase. {Additionally, we also show results of node classification in the setting of WalkLM in Table~\ref{tab:walklmsetting}}.

\begin{table*}[t]
\small
   \renewcommand\arraystretch{1.1}
   \setlength{\tabcolsep}{6pt}
   \centering
   \caption{Ablation study: Experimental results of different path structures involving 1-hop neighbors, random walks (RW), ($\alpha,\beta,\gamma$) RW, randomly sampled short shortest paths (Random Short) and long shortest paths.}
   \label{tab:path-structure}
   \begin{tabular}{lcccccccc}\hline
   Datasets & \multicolumn{4}{c}{PubMed} & \multicolumn{4}{c}{Cora} \\
   \cmidrule(lr){2-5} \cmidrule(lr){6-9}
   Tasks &  \multicolumn{2}{c}{Node classification}  &   \multicolumn{2}{c}{Edge Validation} & \multicolumn{2}{c}{Node classification}  &   \multicolumn{2}{c}{Edge Validation}  \\
   \cmidrule(lr){2-3} \cmidrule(lr){4-5}  \cmidrule(lr){6-7} \cmidrule(lr){8-9}
    Metrics & Macro-F1 & Micro-F1 & AUC & Accuracy & Macro-F1 & Micro-F1 & AUC & Accuracy  \\ \hline
    LLM &  0.716 \textcolor{blue}{\footnotesize(-3.1\%)}  &  0.724\textcolor{blue}{\footnotesize(-3.1\%)}   &  0.714\textcolor{blue}{\footnotesize(-3.5\%)} &  0.666\textcolor{blue}{\footnotesize(-4.5\%)} & 0.660\textcolor{blue}{\footnotesize(-9.2\%)}  &  0.694\textcolor{blue}{\footnotesize(-8.3\%)}  & 0.856\textcolor{blue}{\footnotesize(-6.8\%)} &  0.780\textcolor{blue}{\footnotesize(-6.7\%)} \\
    
    LLM+1hop &  0.721 \textcolor{blue}{\footnotesize(-2.6\%)}  &  0.737\textcolor{blue}{\footnotesize(-1.8\%)}   & 0.653\textcolor{blue}{\footnotesize(-9.6\%)} &  0.612\textcolor{blue}{\footnotesize(-9.9\%)} & 0.673\textcolor{blue}{\footnotesize(-7.9\%)}  &  0.704\textcolor{blue}{\footnotesize(-7.1\%)} & 0.870\textcolor{blue}{\footnotesize(-5.4\%)} &  0.791\textcolor{blue}{\footnotesize(-5.6\%)} \\ 
    
    LLM+RW  & 0.723 \textcolor{blue}{\footnotesize(-2.4\%)}  &  0.735\textcolor{blue}{\footnotesize(-2.0\%)}  & 0.642\textcolor{blue}{\footnotesize(-10.7\%)}  & 0.604\textcolor{blue}{\footnotesize(-10.7\%)} & 0.668\textcolor{blue}{\footnotesize(-8.4\%)}  & 0.701\textcolor{blue}{\footnotesize(-7.8\%)} & 0.872\textcolor{blue}{\footnotesize(-5.2\%)}  & 0.789\textcolor{blue}{\footnotesize(-5.8\%)} \\
    
    LLM+($\alpha,\beta,\gamma$)RW~\cite{abcrw} & 0.721 \textcolor{blue}{\footnotesize(-2.6\%)}  &  0.737\textcolor{blue}{\footnotesize(-1.8\%)}  & 0.663\textcolor{blue}{\footnotesize(-8.6\%)}  & 0.624\textcolor{blue}{\footnotesize(-8.7\%)}  & 0.675\textcolor{blue}{\footnotesize(-7.7\%)}  &  0.706\textcolor{blue}{\footnotesize(-7.1\%)}  & 0.875\textcolor{blue}{\footnotesize(-4.9\%)} &  0.798\textcolor{blue}{\footnotesize(-4.9\%)}\\
    
    LLM+Random Short  & 0.728 \textcolor{blue}{\footnotesize(-1.9\%)}  & 0.742\textcolor{blue}{\footnotesize(-1.3\%)}  & 0.655\textcolor{blue}{\footnotesize(-9.4\%)}  & 0.614\textcolor{blue}{\footnotesize(-9.7\%)} & 0.674\textcolor{blue}{\footnotesize(-7.8\%)}  &  0.708\textcolor{blue}{\footnotesize(-6.9\%)} & 0.886\textcolor{blue}{\footnotesize(-3.8\%)}  & 0.806\textcolor{blue}{\footnotesize(-4.1\%)} \\
    
    LLM+Long SP  & 0.729 \textcolor{blue}{\footnotesize(-1.8\%)}  &  0.744\textcolor{blue}{\footnotesize(-1.1\%)} & 0.658\textcolor{blue}{\footnotesize(-9.1\%)}  &  0.619\textcolor{blue}{\footnotesize(-9.2\%)} & 0.669\textcolor{blue}{\footnotesize(-8.3\%)}  &  0.705\textcolor{blue}{\footnotesize(-7.2\%)} & 0.877\textcolor{blue}{\footnotesize(-4.7\%)}  & 0.800\textcolor{blue}{\footnotesize(-4.7\%)} \\
     \hline
    \textbf{LLM + L2SP}  & \textbf{0.747}  & \textbf{0.755} & \textbf{0.749}  &  \textbf{ 0.711} & \textbf{0.752}  & \textbf{0.777} & \textbf{0.924}  &  \textbf{0.847} \\
  \hline
    \end{tabular}
\end{table*}

\begin{figure}[t]
\centering
  \includegraphics[width=0.38\textwidth]{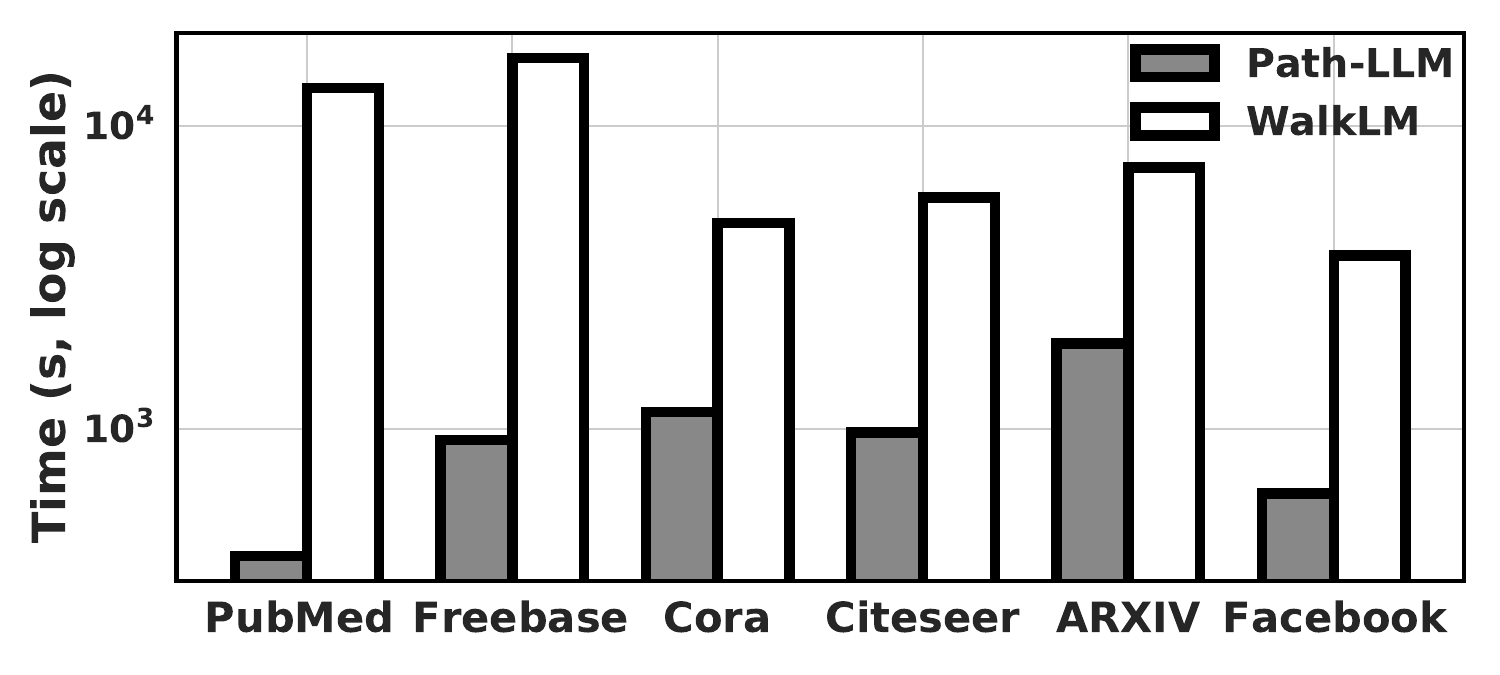}
  \caption{Efficiency evaluation. Training time comparison of Path-LLM and WalkLM across six datasets (PubMed, Freebase, Cora, Citeseer, ARXIV, and Facebook).}
  \label{fig:efficiency}
\end{figure}

\stitle{Exp-2: Effectiveness on edge validation task.}
The edge validation task involves determining whether an edge exists between two given nodes (refer to link prediction in~\cite{walklm}\cite{abcrw}\cite{tang2024graphgpt}\cite{liu2024one}\cite{chen2024llaga}). Formally, given two nodes $u$ and $v$, verify whether $(u,v)\in E$ or $(u,v)\notin E$. We verify the existence of $(u,v)$ based on the learned graph embeddings.
Specifically, we use the Hadamard function to construct feature vectors for node pairs and train a two-layer MLP classifier on the selected links. We evaluate graph embeddings with AUC (area under the ROC curve) and Accuracy~\cite{auc}. Note that for GNNs, we use GNNs to generate graph embeddings and then construct feature vectors for node pairs.
Table~\ref{tab:link} shows that our Path-LLM achieves remarkable performance in uncovering latent associations among nodes in text-attributed graphs. This implies that Path-LLM possesses a more precise grasp of graph structure through semantic information integration, thus conferring advantages in edge validation tasks.
Path-LLM consistently outperforms WalkLM,
achieving a notable 25.74\% relative improvement in AUC and 25.11\% in Accuracy on PubMed. Meanwhile, Path-LLM achieves a notable 7.73\% relative improvement in AUC and 9.42\% in Accuracy over WalkLM on Cora. On the large-scale citation network ARXIV, Path-LLM shows a relative performance gain of 9.73\% in AUC and 14.35\% in Accuracy over WalkLM.
Compared to pure LLM, the structural advantages displayed by the \ltos on PubMed have an average 5.82\% increase and a 7.90\% increase in AUC on Cora. 
Table~\ref{tab:link} shows that our Path-LLM self-supervised method of learning graph structure from L2SP is significantly effective, outperforming GraphGPT~\cite{tang2024graphgpt} on all six datasets, with an average relative gain of 4.54\% in AUC and 6.39\% in Accuracy. {Additionally, we also show results of edge validation in the setting of WalkLM in Table~\ref{tab:walklmsetting}}.



\stitle{Exp-3: Training duration across six datasets.}
We compare the training times of WalkLM and Path-LLM. Even though WalkLM utilizes the smaller language model distill-roberta-66M, Path-LLM demonstrates significantly faster performance on various graphs, as illustrated in Figure~\ref{fig:efficiency}. 
On average, Path-LLM is 12 times faster than WalkLM across these datasets and 11 times faster across millions-scale graphs. Notably, on PubMed, Path-LLM takes only 380 seconds, while WalkLM takes 13,370 seconds, making Path-LLM approximately 35 times faster. {Based on the analysis of complexity, two main reasons are: (1) the average text length $\hat{S}$ in Path-LLM is much smaller than $\hat{S}$ in WalkLM, and (2) the number of training paths $|\mathbb{P}_{short}|$ in Path-LLM is much less than $|N|$ in WalkLM.}
Additionally, the number of L2SP paths employed for training Path-LLM is significantly lower than the number of Random Walk paths used for training WalkLM, resulting in an average reduction of 91.09\% in training data across six datasets. 
Efficiency results shown in Figure~\ref{fig:efficiency} verify that the average text length is a critical factor affecting training time, represented by $\hat{S}$ in the time complexity $O(B_{LLM}(I,|\mathbb{P}_{short}|,\hat{S},d))$ of the Path‐LLM training process. L2SP‐based texts in PubMed are much shorter than those in the other three citation networks. As a result, although there are more L2SP paths in PubMed than in Cora and Citeseer, the model training is significantly faster in PubMed.

\begin{table}[t]
\small
   \centering
   \renewcommand\arraystretch{1}
   \caption{Results in the setting of WalkLM.}
   \label{tab:walklmsetting}
   \begin{tabular}{lcccc}\hline
   Datasets  &   \multicolumn{4}{c}{PubMed}  \\
   \cline{2-5}
    Tasks &  \multicolumn{2}{c}{Node classification}  &   \multicolumn{2}{c}{Edge validation}   \\
    \cmidrule(lr){2-3} \cmidrule(lr){4-5}
    Metrics  & Macro-F1 & Micro-F1 & AUC & Accuracy \\ \hline
    WalkLM &  0.6044 & 0.6210 & 0.8359 & 0.7754 \\
    Path-LLM &  0.7622 & 0.7674 & 0.8640 & 0.7983 \\
  \hline
    \end{tabular}
\end{table}

\begin{table}[t]
\small
   \centering
   \renewcommand\arraystretch{1}
   \caption{Ablation study of hyperparameters $L$, when $k=10$.}
   \label{tab:l}
   \begin{tabular}{lcccc}\hline
   Datasets  &   \multicolumn{4}{c}{PubMed}  \\
   \cline{2-5}
    Tasks &  \multicolumn{2}{c}{Node classification}  &   \multicolumn{2}{c}{Edge validation}   \\
    \cmidrule(lr){2-3} \cmidrule(lr){4-5}
    Metrics  & Macro-F1 & Micro-F1 & AUC & Accuracy \\ \hline
    $L=4$ &  0.7285 & 0.7422 & 0.6829 & 0.6424 \\
    $L=6$ &  0.7308 & 0.7445 & 0.7085 & 0.6734 \\
    $L=8$ &  $\colorbox{Gainsboro}{\text{0.7372}}$ & $\colorbox{Gainsboro}{\text{0.7466}}$ & \textbf{0.7560} & \textbf{0.7104}  \\
    $L=10$ &  \textbf{0.7501} & \textbf{0.7532} & $\colorbox{Gainsboro}{\text{0.7362}}$ & $\colorbox{Gainsboro}{\text{0.6941}}$ \\
    $L=12$ &  0.7277 & 0.7357 & 0.6995& 0.6595 \\
  \hline
    \end{tabular}
\end{table}

\begin{table}[t]
\small
   \centering
   \renewcommand\arraystretch{0.9}
   \caption{Ablation study of hyperparameters $k$, when $L=10$.}
   \label{tab:k}
   \begin{tabular}{lcccc}\hline
   Datasets  &   \multicolumn{4}{c}{PubMed}  \\
   \cline{2-5}
    Tasks &  \multicolumn{2}{c}{Node classification}  &   \multicolumn{2}{c}{Edge validation}   \\
    \cmidrule(lr){2-3} \cmidrule(lr){4-5}
    Metrics  & Macro-F1 & Micro-F1 & AUC & Accuracy \\ \hline
    $k=1$ &  \textbf{0.7529} & \textbf{0.7644} & 0.6866 & 0.6451 \\
    $k=3$ &  0.7334 & 0.7467 & 0.6967 & 0.6571 \\
    $k=5$ &  0.7471 & $\colorbox{Gainsboro}{\text{0.7555}}$ & $\colorbox{Gainsboro}{\text{0.7497}}$ & \textbf{0.7111} \\
    $k=7$ &  0.7383 & 0.7511 & 0.7453 & 0.7063 \\
    $k=10$ & $\colorbox{Gainsboro}{\text{0.7501}}$ & 0.7532 & 0.7362 & 0.6941 \\
    $k=12$ & 0.7186 & 0.7379 & \textbf{0.7510} & $\colorbox{Gainsboro}{\text{0.7059}}$ \\
  \hline
    \end{tabular}
\end{table}

\stitle{Exp-4: Ablation studies.}
We perform an ablation study 
to verify the analysis in Section~\ref{text-with-properties}, demonstrating the effectiveness of our proposed \ltos structure
compared to other path structures. 
To conduct a comprehensive evaluation, we propose several baselines for ablation study: LLM with 1-hop neighbors, LLM with random walk, LLM with new $(\alpha, \beta, \gamma)$ random walk~\cite{abcrw}, LLM with long shortest path and LLM with randomly sampled short shortest path. Particularly, when reproducing $(\alpha, \beta, \gamma)$ random walk, we utilize Sentence-BERT~\cite{reimers2019sentence} to process node embeddings. Here, we present our findings of \ltos-based shortest paths and Path-LLM from PubMed, Freebase, Cora, ARXIV, Citeseer and Facebook.



\begin{figure}[t]
\centering
  \includegraphics[width=0.45\textwidth]{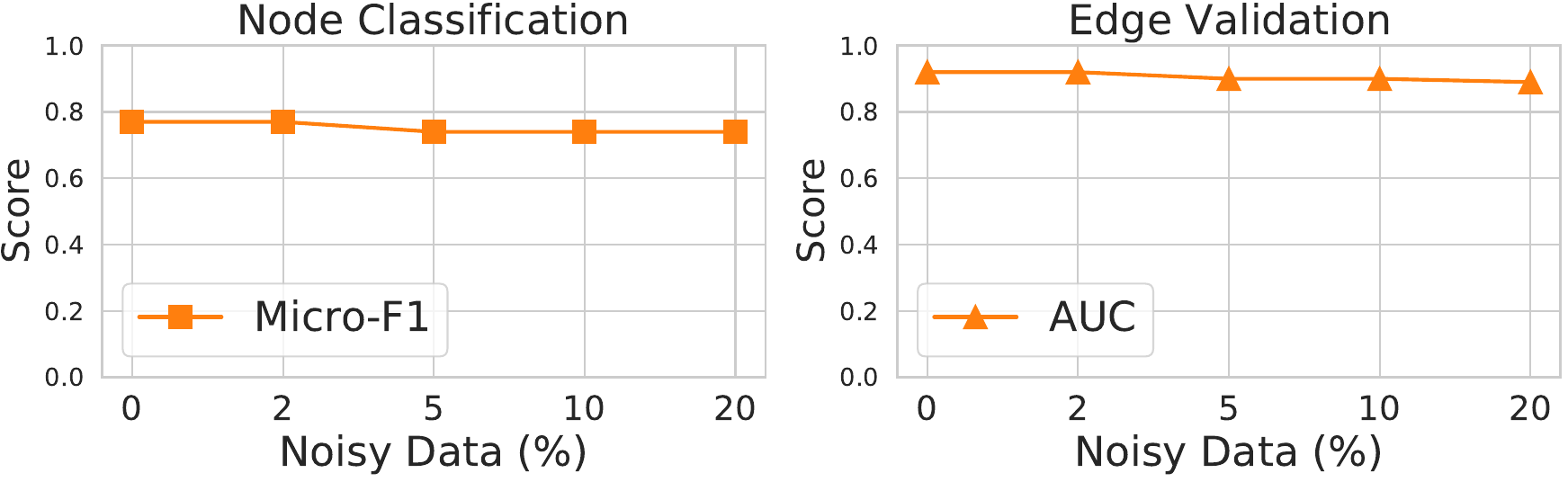}
  \caption{Robustness evaluation of Path-LLM.}
  \label{fig:robust}
\end{figure}

\begin{figure}[t]
\centering
  \includegraphics[width=0.47\textwidth]{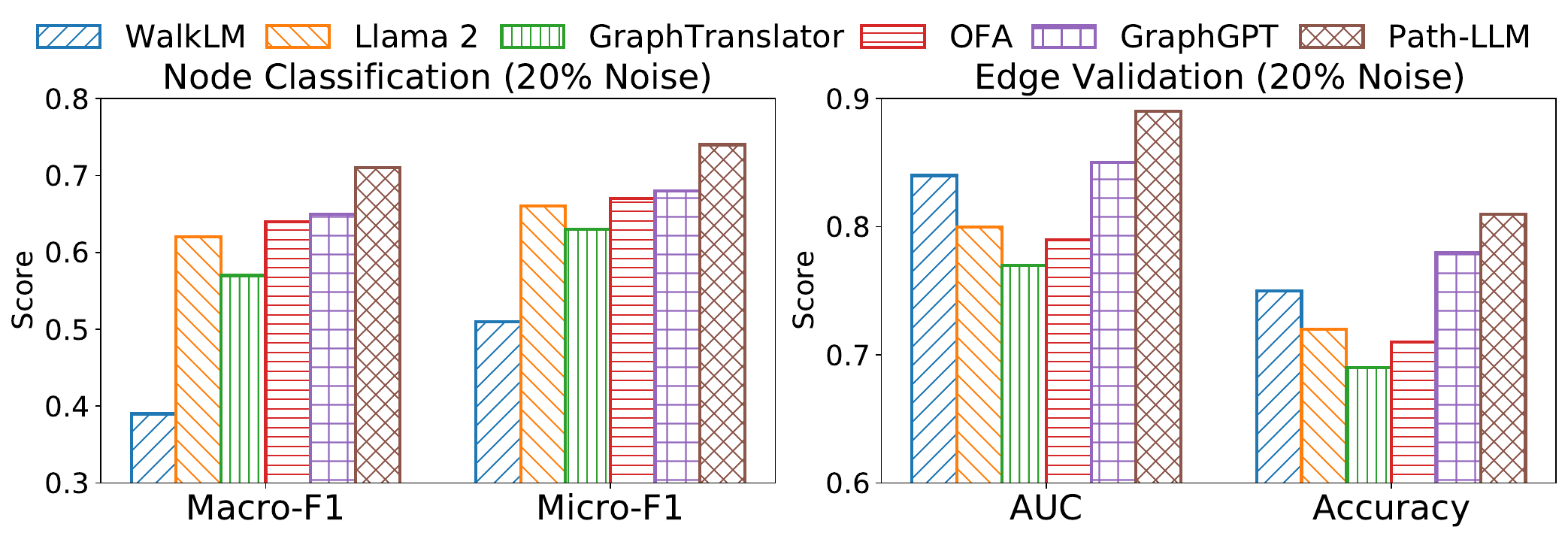}
  \caption{Performance comparison of Path-LLM and other LLM-based methods after adding 20\% noisy data in Cora.}
  \label{fig:robust comparison}
\end{figure}

\begin{table}[t]
\small
   \centering
   \renewcommand\arraystretch{1}
   \caption{Performance comparison of different LLMs on node classification (NC) and edge validation (EV).}
   \label{tab:different llm}
   \begin{tabular}{lcccc}\hline

   Datasets & \multicolumn{2}{c}{PubMed} & \multicolumn{2}{c}{Cora} \\ 
   \cline{2-3}  \cline{4-5}
   Tasks &  \multicolumn{1}{c}{NC}  &   \multicolumn{1}{c}{EV} & \multicolumn{1}{c}{NC}  &   \multicolumn{1}{c}{EV}  \\ 
   \cmidrule(lr){2-3} \cmidrule(lr){4-5}  
    Metrics &  Micro-F1 & AUC  & Micro-F1 & AUC   \\ \hline
   Llama 1~\cite{touvron2023llama} &  0.7158 & 0.6887  & 0.6990 & 0.8439\\
   Llama 1+\ltos & \textbf{0.7445} & \textbf{0.7322} & \textbf{0.7647} & \textbf{0.9179} \\ \hline
    Llama 2~\cite{llama2} & 0.7246  &  0.7144 &  0.6946  & 0.8568  \\ 
    \textbf{Path-LLM (Ours)}  & \textbf{0.7555} &    \textbf{0.7497}   & \textbf{0.7773} & \textbf{0.9244}  \\
  \hline
    \end{tabular}
\end{table}

\begin{table}[h]
    \centering
    \small
    \setlength{\tabcolsep}{8pt}
    \renewcommand\arraystretch{1.05}
    \caption{Results of average answer distance among query keywords. Lower distance shows better performance.}
    \begin{tabular}{cccc}
        \hline
        Keywords Number& $|Q|=2$ & $|Q|=3$ & $|Q|=4$ \\
        \hline
        uniform weight & 4.0 & 6.16 & 8.73 \\
        WalkLM~\cite{walklm} & 1.27 & 1.91 & 2.72 \\
        GraphGPT~\cite{tang2024graphgpt} & 1.24 & 1.87 & 2.63 \\
        \hline
        Path-LLM & \textbf{0.93} & \textbf{1.46} & \textbf{2.04} \\
        \hline
    \end{tabular}
    \label{tab:keyword_comparison}
\end{table}

\stitle{Ablation-study-1: \ltos-based shortest paths outperform random walks.}
Table~\ref{tab:path-structure} indicates that the \ltos-based shortest path structure outperforms the random walk structure on node classification and edge validation. We compare \ltos-based shortest paths with two random walk algorithms, which are the random walk in WalkLM and the new $(\alpha,\beta,\gamma)$ random walk. 
Path-LLM is superior to the new $(\alpha,\beta,\gamma)$ random walk with an average increase of 3.03\% on PubMed and 10.74\% on Cora for node classification. While in edge validation, RW-based LLM even decreased by 10.08\% of AUC after training with random walks compared to pure LLM. 
Our \ltos-based LLM outperforms the RW-based LLM with an average increase of 17.19\% on PubMed and 6.66\% on Cora, demonstrating that \ltos\ represents a better path for capturing graph structure compared to Random Walk.

\begin{figure*}[t]
\centering
   \includegraphics[width=0.85\textwidth]{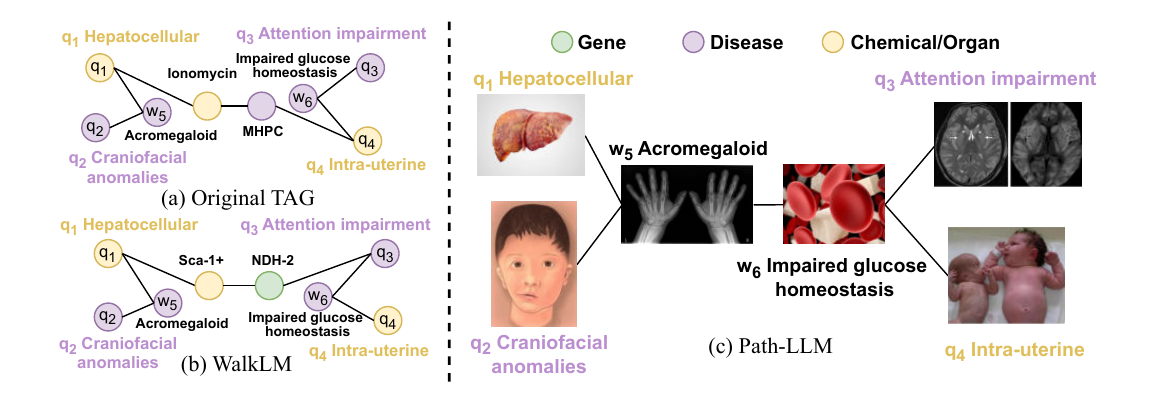}
  \caption{Case study of keyword search with a query $Q=\{q_1$: hepatocellular, $q_2$: craniofacial anomalies, $q_3$: attention impairment, $q_4$: intra-uterine$\}$.  Path-LLM finds a semantically meaningful answer with the smallest number of five edges. 
  }
  \label{fig:case}
\end{figure*}

\stitle{Ablation-study-2: \ltos-based shortest paths outperform other different path structures.}
Table~\ref{tab:path-structure} also compares the effectiveness of different graph embeddings with distinct path structures on node classification and edge validation.
It is evident that the \ltos-based shortest path structure outperforms all other path structures. 
Especially in edge validation, which evaluates the graph structure understanding of LLMs, the results show that different path structures capture different graph features, which influence LLMs' learning of the graph structure. Our proposed \ltos-based structure yields the best results among 
all other different path structures, with 12.31\% performance gains in AUC and 12.95\% accuracy over randomly sampled shortest paths on PubMed while an average of 5.62\% performance gains over original long shortest paths on Cora.

\stitle{Ablation-study-3: Different hyperparameters analysis.}
We conduct ablation studies varying hyperparameters $L$ and $k$, testing $L$ in the interval [4,12] and $k$ in the interval [1,12]. 
$L$ represents the minimum length of the long shortest path in the L2SP selection phase. $k$ denotes the number of random samples of long shortest paths with the same source and target nodes in the L2SP selection phase.
Results shown in Table~\ref{tab:l} show that Path‐LLM performs well in downstream tasks for $L$ in [8,10], and we choose $L=10$ in experiments. Table~\ref{tab:k} shows that $k=5$ is a good trade-off choice.

\stitle{Ablation-study-4: Our Path-LLM shows strong robustness even after adding 20\% noisy data.} For each text attribute in Cora, we add 2\%-20\% noisy and irrelevant words in random positions. We then use Path-LLM to generate embeddings based on noisy Cora. Results in Figure~\ref{fig:robust} show the strong robustness of our Path-LLM with a slight decrease of 4\%. In addition, we compare with other LLM-based baselines on Cora after adding 20\% noisy data. Results in Figure~\ref{fig:robust comparison} show that Path-LLM outperforms all baselines, once again verifying the strong robustness of Path-LLM.

\stitle{Ablation-study-5: Our \ltos-based self-supervised learning consistently enhances the performance of different LLMs.}
In Table~\ref{tab:different llm}, we evaluate Llama 1 and Llama 2. After L2SP-based self-supervised training, LLMs exhibit an average absolute improvement of 3\% in node classification and 4-5\% in edge validation.

\begin{figure}[t]
  \includegraphics[width=0.46\textwidth]{CORA.pdf}
  \caption{Embedding visualization of Cora.}
  \label{fig:tsne_cora}
\end{figure}

\begin{figure}[t]
  \includegraphics[width=0.46\textwidth]{PUBMED.pdf}
  \caption{Embedding visualization of PubMed.}
  \label{fig:tsne_pubmed}
\end{figure}

\stitle{Exp-5: Keyword search evaluation.} 
We quantitatively evaluate the effectiveness of keyword search. We use the sum distance among given keywords as our evaluation metric. The lower the distance, the closer the connection among keywords, and the better the result. We test three groups of keyword queries on PubMed, i.e., 2-keyword, 3-keyword, and 4-keyword. For each group, we randomly test 100 queries and get the average result. The results in Table~\ref{tab:keyword_comparison} show that Path-LLM can find answers with the smallest distance. 

\stitle{Exp-6: Case study of keyword search.}
We conduct a case study of keyword search by comparing three graph weighting methods, including the uniform weights, WalkLM-based weights, and our Path-LLM-based weights. This aims to evaluate the effectiveness of Path-LLM weights to find tight groups in terms of graph structure and node semantics.
Our comprehensive case study thoroughly explores the inherent benefits of Path-LLM concerning both graph structure and semantics on PubMed. 
Figure~\ref{fig:case} shows Steiner trees based on four given keywords $Q=\{q_1$: hepatocellular, $q_2$: craniofacial anomalies, $q_3$: attention impairment, $q_4$: intra-uterine$\}$. 
From the graph structure perspective, the Path-LLM-based weight obtains the fewest five edges and the closest ties between nodes. In contrast, graphs generated with uniform and WalkLM-based weights are sparser, with longer and less tight connections between nodes, like the association between the nodes $\{w_5$: acromegaloid$\}$ and $\{w_6$: impaired glucose homeostasis$\}$. From the node semantics perspective, the Path-LLM-based weighted subtree contains rich and medically proven associations. To elaborate, the occurrence of $(q_1)$ hepatocellular issues along with symptoms of $(q_2)$ craniofacial anomalies is highly indicative of the presence of $(w_5)$ acromegaly \cite{carmichael2017association, akirov2021biochemical}. $(w_5)$ may lead to $(w_6)$ impaired glucose homeostasis \cite{carmichael2009utility, chang1992diagnosis} with potential risks of $(q_3)$ attention impairment caused by brain damage \cite{auer1986progress, su2012research, kalra2013hypoglycemia} and $(q_4)$ intra-uterine growth retardation \cite{langer1986link, ogata1985altered, gddotnuemes2020hyperinsulinemic}. In contrast, subtrees derived from uniform and WalkLM-based weights exhibit weaker semantic associations between nodes and lack strong medical evidence.

\stitle{Exp-7: Visualization of Path-LLM embeddings.}
For an intuitive comparison, we visualize the embedding space of different types of nodes learned by WalkLM and our proposed Path-LLM on Cora and PubMed, respectively. 
The embeddings are further transformed into the 2-dimensional Euclidean space via the t-SNE \cite{tsne}. The nodes are colored according to their types. Figures \ref{fig:tsne_cora} and~\ref{fig:tsne_pubmed} show that node representations derived from our Path-LLM are more discriminative from different classes than WalkLM.

\section{CONCLUSIONS}
\label{sec:conclusion}
In this paper, we propose a novel Path-LLM model for unified graph representation learning. The key design of our model involves a sampled selection of our proposed \ltos-based shortest paths to represent the whole network. Then, we utilize a large language model to integrate graph structure into deep semantic embedding space by learning our \ltos-based texts. 
We develop techniques to construct a new weighted graph based on Path-LLM-based embedding and tackle an NP-hard graph querying task of keyword search by finding better answers.  
Comprehensive experiments validate the effectiveness of our Path-LLM model and embeddings. 



\bibliographystyle{ACM-Reference-Format}
\bibliography{ref}

\end{document}